\pdfoutput=1

\documentclass[11pt]{article}

\usepackage[final]{acl}

\usepackage[markup=underlined]{changes}

\usepackage{times}
\usepackage{latexsym}
\usepackage[utf8]{inputenc}   
\usepackage[T1]{fontenc}    	
\usepackage{url}              		 
\usepackage{booktabs}       	
\usepackage{amsfonts}       	
\usepackage{nicefrac}       	  
\usepackage[final,nopatch=footnote]{microtype}      	

\usepackage{inconsolata}
\usepackage{graphicx}
\usepackage{color}
\usepackage{rotating}
\usepackage{tabularx}
\usepackage{pdflscape}
\usepackage{amsmath,amsthm,amssymb}
\usepackage{algorithm,algorithmic}
\usepackage{bbm,dsfont}
\usepackage{subfig}
\usepackage{caption}
\usepackage{wrapfig}
\usepackage{cancel}
\usepackage{enumerate,cases}
\usepackage{thmtools,thm-restate}
\usepackage{mathtools}

\usepackage{multirow}
\usepackage{makecell}
\usepackage{setspace}
\usepackage{pifont}
\usepackage{array}
\usepackage{enumitem}
\usepackage{lineno}     

\allowdisplaybreaks

\usepackage{silence}
\usepackage{hyperref}		 
\hypersetup{
}
\usepackage[capitalize]{cleveref}

\crefformat{figure}{Fig.~#2#1#3}       
\crefformat{table}{Tab.~#2#1#3}        
\crefformat{equation}{Eq.~(#2#1#3)}    
\crefformat{section}{Sec.~#2#1#3}      
\crefformat{appendix}{App.~#2#1#3}     
\crefformat{algorithm}{Alg.~#2#1#3}    

\newcommand{\squishlisttwo}{
 \begin{list}{$\bullet$}
  { \setlength{\itemsep}{1pt}
     \setlength{\parsep}{0pt}
    \setlength{\topsep}{0pt}
    \setlength{\partopsep}{0pt}
    \setlength{\leftmargin}{1em}
    \setlength{\labelwidth}{1.5em}
    \setlength{\labelsep}{0.5em} } 
}
\newcommand{\squishend}{
  \end{list}  }

\newcommand{\alg}{\textsc{Tetris}}
\newcommand{\vllm}{\texttt{vLLM}}
\newcommand{\sglang}{\texttt{SGLang}}

\usepackage{amsmath,amsfonts,bm}

\def\eqref#1{equation~\ref{#1}}

\def\1{\bm{1}}

\def\vone{{\bm{1}}}

\DeclareMathAlphabet{\mathsfit}{\encodingdefault}{\sfdefault}{m}{sl}
\SetMathAlphabet{\mathsfit}{bold}{\encodingdefault}{\sfdefault}{bx}{n}

\def\gB{{\mathcal{B}}}

\def\gD{{\mathcal{D}}}

\def\gG{{\mathcal{G}}}
\def\gH{{\mathcal{H}}}

\def\gM{{\mathcal{M}}}

\def\gO{{\mathcal{O}}}

\def\gS{{\mathcal{S}}}

\theoremstyle{plain}

\newtheorem{lem}{Lemma}

\newtheorem{assu}{Assumption}

\title{\includegraphics[width=0.8\baselineskip]{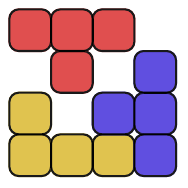} \alg{}: Optimal Draft Token Selection for Batch Speculative Decoding}

\author{
 \textbf{Zhaoxuan Wu\textsuperscript{*1}},
 \textbf{Zijian Zhou\textsuperscript{*1,2}},
 \textbf{Arun Verma\textsuperscript{1}},
 \textbf{Alok Prakash\textsuperscript{1}},
\\
 \textbf{Daniela Rus\textsuperscript{1,3}},
 \textbf{Bryan Kian Hsiang Low\textsuperscript{1,2}}
\\
\\
 \textsuperscript{1}Singapore-MIT Alliance for Research and Technology, Republic of Singapore
 \\
 \textsuperscript{2}Dept. of Computer Science, National University of Singapore, Republic of Singapore \\
 \textsuperscript{3}CSAIL, Massachusetts Institute of Technology, USA
\\
 \small{
   \textbf{Correspondence:} \href{mailto:lowkh@comp.nus.edu.sg}{lowkh@comp.nus.edu.sg}
 }
}

\begin{document}
    \makeatletter
    \def\blankfootnote{\xdef\@thefnmark{}\@footnotetext}
    \makeatother
    
    \maketitle
    
    \begin{abstract}
        We propose \alg{}, a novel method that optimizes the {\em total throughput} of batch speculative decoding in multi-request settings. 
        Unlike existing methods that optimize for a single request or a group of requests as a whole, \alg{} actively selects the most promising draft tokens (for every request in a batch) to be accepted when verified in parallel, resulting in fewer rejected tokens and hence less wasted computing resources. Such an effective resource utilization to achieve fast inference in large language models (LLMs) is especially important to service providers with limited inference capacity.
        Compared to baseline speculative decoding, \alg{} yields a consistently higher acceptance rate and more effective utilization of the limited inference capacity. 
        We show theoretically and empirically that \alg{} outperforms baseline speculative decoding and existing methods that dynamically select draft tokens, leading to a more efficient batch inference in LLMs.
        \blankfootnote{\textbf{*} Equal contribution.}
    \end{abstract}

    \section{Introduction}
    \label{sec:intro}

Transformer-based large language models (LLMs) have shown remarkable abilities to solve different tasks across various domains, such as natural language~\citep{zhao2023survey}, computer vision~\citep{Yin2024mllm}, robotics~\citep{zeng2023largelanguagemodelsrobotics}, code generation~\citep{roziere2024codellamaopenfoundation}, among others~\citep{maslej2024ai-report}. 
However, the autoregressive nature of LLMs (i.e., generating one token at a time) leads to an increasingly sluggish inference speed as the model size increases.

To address this problem, a recent widely-used approach is speculative decoding (SD)~\citep{cai2024medusa,cheng2024recurrentdrafterfastspeculative,leviathan2023,li2024eagle2,li2024eagle}: It achieves faster inference by using a \textit{small draft model} to rapidly generate a sequence of {\em (draft) tokens} and then a {\em large target model} to verify whether to accept or reject them in parallel. 
When a token is rejected, the draft model generates a new sequence of tokens in the next step, starting from the most recently accepted token.
A key aspect of SD is to determine the optimal number of draft tokens (i.e., {\em draft window size}) to generate and verify in each step.
Generating more draft tokens allows the target model to verify a longer sequence at once (given sufficient computing resources/capacity for parallel inferences), which can potentially boost inference speed. 
However, doing so increases the risk of wasting computing resources 
since all tokens following the first rejected token must be discarded.
In contrast, generating fewer draft tokens reduces this risk but limits the potential benefit of SD since the computing resources are not effectively utilized.
Therefore, the optimal selection of draft tokens that would be accepted when verified by the target model in parallel is critical to improving both inference speed and resource utilization~\citep{liu2024optimizingspeculativedecodingserving}.

\begin{figure*}[t]
    \centering
    \includegraphics[width=0.65\linewidth]{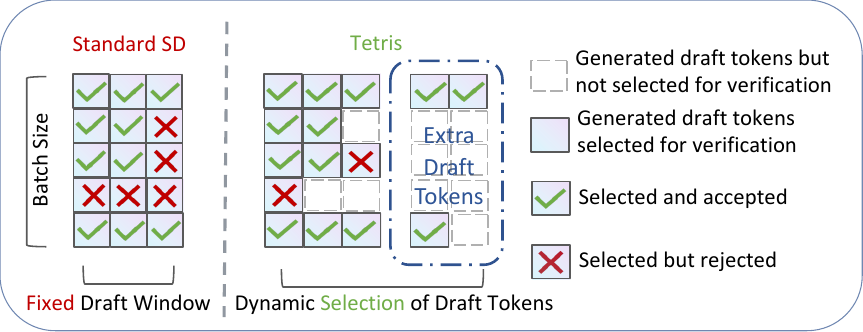}
    \caption{
        \textcolor{red!60!black}{Standard SD} (left) uses a fixed draft window size, while \textcolor{green!50!black}{\alg{}} (right) generates extra draft tokens and dynamically optimizes draft token selection for every request in a batch, resulting in more accepted tokens. 
    }
    \label{fig:tetris_illustration}
\end{figure*}

Most existing works have focused on optimizing draft token selection for individual user requests~\citep{agrawal2024adaedlearlydraftstopping,huang2024specdecboostingspeculativedecoding,liu2024parallelspeculativedecodingadaptive,mamou2024dynamicspeculation}, but 
may not work well for profit-driven LLM inference service providers who must manage multiple user requests under a limited inference capacity.
Moreover, LLM inference service providers typically charge users based on the number of tokens served~\citep{fireworksai,replicate}. 
Hence, they are incentivized to maximize the total number of tokens served (i.e., {\em throughput}) across all user requests while ensuring fast response time to meet service level agreement~\citep{wieder2011service}.
So, they would employ computing clusters to process large batches of user requests simultaneously and use SD to further improve the inference speed.

Such batch processing of user requests entails a fundamentally different optimization objective for SD compared to handling individual requests.
For SD of a single request, supposing a fast draft model with negligible runtime, the objective is to maximize the draft window size as long as the target model can verify all draft tokens in parallel by fully utilizing the inference capacity.
It can be naively extended to batch processing by 
widening the draft window for all requests until the inference capacity is reached.
This is inefficient as each request may require a different optimal draft token selection due to varying 
difficulty in speculation (i.e., generating tokens to match the target model's outputs).

This paper presents a theoretical framework that dynamically optimizes the draft token selection for every user request from the perspective of a capacity-limited LLM inference service provider who aims to maximize resource utilization. 
Since draft token verification is the most time-consuming component of SD, we propose \textbf{\alg}, a method that greedily selects draft tokens with a high likelihood of acceptance by the target model. The name of our method is derived from the shape of its selected tokens, as shown in~\cref{fig:tetris_illustration}.
We demonstrate that 
\alg{} strictly outperforms standard SD by achieving higher total throughput.
Our work bridges a critical yet overlooked gap in current research, allowing service providers to improve total throughput with batch SD. 
The specific contributions of our work here are summarized below:
\squishlisttwo
    \item In \cref{sec:problem}, we introduce the problem of optimal draft token selection in multi-request settings, and in \cref{sec:tetris}, we propose \alg{}, a novel method that selects optimal draft tokens in log-linear time for the target model's verification. 

    \item In \cref{sec:analysis}, we theoretically show that \alg{} achieves optimal throughput at each decoding step and globally in the absence of drafting time (i.e., time to generate draft tokens) under reasonable token acceptance assumptions.

    \item In \cref{sec:experiment}, our empirical results show that \alg{} consistently outperforms standard SD and existing methods that use dynamic draft windows for a batch in terms of total throughput and end-to-end latency (including drafting time),
    highlighting the potential of \alg{} to improve inference speed in real-world model service deployments.
\squishend

    \section{Related Work}
    \label{sec:related_work}

\paragraph{Speculative Decoding (SD).}
By employing a draft-then-verify strategy for lossless accelerations of LLM inference, SD has attracted significant attention recently~\citep{ryu2024closerlook-survey,xia2024sd-survey}.
Recent advancements based on SD have focused on developing more efficient draft models by producing multiple drafts for the next few tokens~\citep{cai2024medusa,cheng2024recurrentdrafterfastspeculative,li2024eagle}.
Additionally, some methods have optimized the speculation accuracy by aligning the draft model with the target model~\citep{liu2024onlinespeculativedecoding,zhou2024distillspec} or leveraging the target model itself to draft via techniques like layer skipping~\citep{zhang2024skiplayer}.
To facilitate more efficient verification, tree attention has been proposed 
for speedy tree-structured candidate verification~\citep{miao2024specinfer,spector2023stagedSD}.
In contrast, our work explores a complementary approach that intervenes between the draft and target models, performing strategic draft token selection to improve throughput over batched requests.
Our method can be seamlessly integrated with the above techniques for a more efficient SD system.

\paragraph{LLM Scheduling.}
With the growing popularity of LLM as a service, several works have considered improvements to the scheduling of LLM services. 
These works can be broadly categorized into client-side and server-side approaches.
Server-side approaches~\citep{fu2024efficient,kim2024accelerating,liu2024optimizingspeculativedecodingserving,wang2024opttreespeculativedecodingadaptive} have focused on increasing the throughput of LLM services, which may lead to an unfair allocation of inference resources to users, hence causing starvation. 
On the other hand, client-side approaches~\citep{liu2024andes,sheng2024fairness} have focused on improving user satisfaction by improving client-side metrics (e.g., decreasing maximal waiting time or end-to-end latency).
Our work considers the scenario where the LLM inference service provider employs SD to ensure user satisfaction with inference speed while simultaneously aiming to maximize service throughput to optimize profitability.

\paragraph{Draft Window Optimization.}
In the foundational paper on SD, the authors have proposed to generate a window of draft tokens~\citep{leviathan2023}. 
The optimal draft window is theoretically determined under an impractical assumption of identical conditional acceptance rates for all draft tokens~\citep{leviathan2023}.
Empirically, such an acceptance rate can be estimated by a moving average of past requests~\citep{liu2024optimizingspeculativedecodingserving}.
Other heuristics for finding the optimal draft window include stopping the draft generation when the draft model's confidence score falls below a predetermined threshold~\citep{kim2023speculativedecodingbiglittle,liu2024kangaroo} or when an entropy-controlled criterion is met~\citep{agrawal2024adaedlearlydraftstopping}.
\citet{cai2024medusa} have proposed taking the union of these two heuristics.
These existing works have operated at a single-request level, except that of~\citet{liu2024optimizingspeculativedecodingserving}  which adaptively determines a single draft window for all requests in a batch.
Note that considering each request independently or using a common draft window for a batch
can lead to inefficiencies in allocating verification budgets (i.e., inference capacity) across multiple requests, especially when operating under the limited computing resources of an LLM inference service provider.

    \section{Problem Setup}
    \label{sec:problem}
This section first introduces speculative decoding and then describes the optimal draft token selection problem and the performance metrics used.

\subsection{Speculative Decoding (SD)}
SD is an efficient inference method designed to accelerate the decoding process in LLMs and 
involves two phases: drafting followed by verification.
Initially, a lightweight draft model, denoted as $\gS$, quickly generates candidate draft tokens.
Subsequently, these tokens are verified against the generations from the target model, denoted as $\gM$, which is also often referred to as the verification model.
SD allows parallelized verifications of tokens by $\gM$, as opposed to the conventional autoregressive decoding used in language models.
Hence, SD yields significant improvement in decoding speed.

Specifically, the draft model generates $k$ draft tokens $d_1, \ldots, d_k$ in an autoregressive manner  where $k$ is the draft window size.
Given a prompt or prefix $x$, the generation process follows $d_i \sim p_\gS(\cdot|x, d_1, \ldots, d_{i-1})$.
For notational simplicity, we denote $p_\gS(d_i) = p_\gS(d_i|x, d_1, \ldots, d_{i-1})$.
The verification follows a rejection sampling procedure.
If $p_\gS(d_i) \leq p_\gM(d_i)$, the draft token $d_i$ is accepted.
Otherwise, we reject the draft token with a probability of $1 - p_\gM(d_i)/p_\gS(d_i)$ and then output a new token sampled from an adjusted distribution $p_\gM(d_i') = \mathrm{norm}(\max(0, p_\gM(d_i') - p_\gS(d_i')))$, where $\mathrm{norm}(\cdot)$ normalizes the probability distribution.
Hence, the acceptance of draft tokens depends on both $p_\gS(\cdot)$ and $p_\gM(\cdot)$ and plays a vital role in the effectiveness of SD.
A higher acceptance suggests the possibility of greater speedup gain with a larger $k$.
We defer a more detailed discussion of the acceptance rate estimation in~\cref{app:acceptance-related-work}. 
However, we highlight that the effectiveness of SD is limited by the computing resources available. 
Using a draft window exceeding the capacity for parallel inferences that the server can manage degrades the performance, which we show empirically later in~\cref{sec:experiment}. 
Consequently, it is essential to carefully \textit{select} the draft window size for each request,
leading to our proposed method outlined next.

\subsection{Optimal Draft Token Selection}
\label{sec:optimal_draft_selection}

We first define a set of other notations used throughout our paper.
We consider a specific LLM inference service provider with a limited capacity $C$, which represents the maximum number of parallel inferences its computing resources can perform. 
The capacity depends on the server configurations of the service provider in practice.
At each time step, the server processes a batch of $N$ requests $r_1, r_2, \cdots, r_N$, each with a partially complete sequence $S_{i,t_i} = (d_{i,1}, \ldots, d_{i, t_i})$ where $t_i$ represents the number of tokens verified/served so far for request $r_i$. 
We allow a variable draft window size $k_i$ for each request $r_i$. 
The draft model $\gS$ drafts a set $\gD \coloneqq \{(i, t) | i \in [N], t \in [t_i+k_i]\}$ such that $|\gD| = \sum_{i=1}^{N} k_i = C$.
For each $(i,t) \in \gD$, we send $S_{i,t}$ to have its \textit{last token} verified by $\gM$. 
We aim to optimally choose the set $\gD$ at each time step to maximize the performance of the server in terms of generation throughput, which we define below.

\paragraph{Per-step Throughput.} 
For each step of SD, we are mainly concerned with maximizing the per-step throughput, i.e., the number of tokens served at each time step. 
Mathematically, let $\vone_{i,t}$ be an indicator variable representing whether the last token of $S_{i,t}$ is accepted, let $\tau_{\text{step}}$ be the time per step. The per-step throughput is then defined as
\begin{equation*}
   \textstyle \gG_{\text{step}} \coloneqq (\mathbb{E}[\sum_{(i,t) \in \gD}\vone_{i,t}] + N ) / \tau_{\text{step}} \ .
\end{equation*}
Note that at least one token is always generated by SD via the \textit{bonus token} mechanism~\citep{leviathan2023}. Thus, without considering drafting time, the throughput of SD is theoretically at least as good as that of autoregressive decoding. 

\paragraph{Total Throughput.}
The total throughput is calculated as the average per-step throughput over a total of $T$ steps with a fixed $\tau_{\text{step}}$ for each step:
\begin{equation*}
    \textstyle \gG \coloneqq T^{-1}\sum_{i=1}^{T} \gG_{\text{step}}\ .
\end{equation*}
Note that it is theoretically difficult to find an optimal draft token selection strategy that maximizes $\gG$ as the relationship between previously verified tokens and the distribution of acceptance rate for the remaining tokens is extremely complex. However, under a mild assumption on token acceptance rate, the optimality of $\gG$ is equivalent to the optimality of $\gG_{\text{step}}$, as explained formally in~\cref{sec:analysis} later.

    \section{\alg{}: Optimal Draft Token Selection}
    \label{sec:draft_window}

In this section, we introduce the details of the \alg{} for batch SD and provide an analysis of its time complexity and optimality.
Overall, we leverage the insight that SD suffers from a cascading failure rate in a single sequence but not across different sequences.
More specifically, we distinguish between two types of tokens involved in drafting: \textit{sequential} and \textit{parallel}. 
For each request $r_i$, all pairs $(i, \cdot) \in \gD$ are sequential, i.e., for all $j < k$, $(i,j)$ must be accepted for $(i, k)$ to be accepted as well, implying a cascade of the failure rate. 
On the other hand, for $i \neq j$, $(i, \cdot)$ and $(j, \cdot)$ are parallel, as the failure rate of $(i, \cdot)$ does not influence that of $(j, \cdot)$. 
We highlight that the distinct nature of the two modes serves as the fundamental motivation of our proposed approach for an improved $\gG_{\text{step}}$, and consequently the total throughput $\gG$.

    \subsection{Our Approach and Design}
    \label{sec:tetris}

We introduce inter-dependencies among requests within a batch.
We favor parallel tokens when selecting sequential tokens leads to an excessive cascading of failure rates, and \textit{vice versa}.
To achieve this, we propose to introduce a manager to actively select the best draft tokens that are most likely to be successfully verified by the target model, thus maximizing the expected number of output tokens.
The manager is integrated into the speculative decoding framework and functions as an intermediary between the draft model and the target model.
It operates on the draft tokens and auxiliary outputs (e.g., token distributions, hidden states) from the draft model and strategically selects those that will be sent for verification by the target model.

At each step, define $p_{i,j}$ the conditional acceptance rate of the token at index $(i,j)$ given its corresponding prefix.
Let $\gB_{i,j} \coloneqq (i, j, \prod_{t=1}^{j} p_{i,t})$ be the tuple containing token indices and the probability of all selected tokens in row $i$ up to $j$ being accepted. 
Instead of simply selecting a fixed window of draft tokens for verification, we \textit{greedily} look for tokens with the highest cumulative acceptance rate $\prod_{t=1}^{j} p_{i,t}$ (and not the standalone acceptance rate $p_{i,j}$).
We let the draft model  propose the \textit{extra draft tokens} beyond the server capacity and then select a set $\gD^*$ of tokens such that it maximally utilizes the compute resource by ensuring $|\gD^*| = C$. 
This process dynamically allocates longer draft windows for requests with ``easy'' tokens and shorter windows for ``hard'' ones, reducing resource wastage while sufficiently leveraging speculation, as illustrated in~\cref{fig:tetris_illustration}. 
\alg{} is outlined in \cref{alg:tetris}.

\begin{algorithm}[h]
    \caption{\alg{}}
    \label{alg:tetris}
    \begin{algorithmic}[1]
        \STATE {\bfseries Input:} draft $\gB$, batch size $N$, capacity $C$
        \STATE Initialize $\gD^* \gets \{\}$, $\gH \gets \text{Heap}()$
        \STATE $Z \gets \text{InitArray}(size=N, value=-1)$
        \FOR{$i \in [N]$}
        \STATE $\gH.\text{insert}(\gB_{i, 0})$
        \ENDFOR
        \REPEAT
        \STATE \texttt{// Dequeue the most probable}
        \STATE $(i,j,p_{ij}) = \gH.\text{extractMax}()$
        \STATE $\gD^* = \gD^* \cup \{(i,j)\}$
        \STATE \texttt{// Record the row-wise frontier}
        \STATE $Z[i]=j$
        \STATE \texttt{// Enqueue new candidates}
        \STATE $\gH.\text{insert}(\gB_{i,j + 1})$
        \UNTIL{$|\gD^*| = C$}
        \STATE {\bfseries return} $\gD^*$
    \end{algorithmic}
\end{algorithm}

\subsection{Analysis}\label{sec:analysis}
We now present our theoretical results, which show the per-step and global optimality of \alg{}.
\begin{restatable}[Per-step Optimality of \alg{}]{thm}{tetris}
    \label{thm:tetris}
    In the absence of drafting time, given the true acceptance rate $\alpha_{i,j}$ of each draft token $(i, j)$, \cref{alg:tetris} produces the optimal per-step throughput defined in \cref{sec:problem}.
\end{restatable}
The proof is delayed to \cref{app:proof_of_local_tetris}. 
While we have established the local optimality of \alg{} in \cref{thm:tetris}, such local optimality does not trivially generalize to maximizing total throughput. 
Nevertheless, we show, in \cref{thm:globalOpt}, that \alg{} is optimal in a 
slightly simpler scenario that retains sufficient complexity of interest.

\begin{assu}
\label{assu:equal_acc}
$\forall j$, all tokens in the $j$-th sequence have an identical acceptance rate denoted as $\alpha_j$.
\end{assu}

\begin{restatable}[Global Optimality of \alg{} under Assumption]{thm}{globalOpt}
    \label{thm:globalOpt}
    Under \cref{assu:equal_acc}, in the absence of drafting time, \alg{} searches for the optimal $\gG$ under the same capacity. Morever, if $\alpha_1 = \alpha_2 = \cdots = \alpha_N$, \alg{} has the same $\gG$ as standard batched speculative decoding.
\end{restatable}

The proof is delayed to \cref{app:thm_global_opt_tetris}.
Overall, we established both per-step and global optimality of \alg{} under theoretical assumptions.
Similar assumptions are commonly made in the literature~\citep{leviathan2023,liu2024optimizingspeculativedecodingserving} to enable theoretical insights (more details in~\cref{app:related-work-assumption}), and \alg{} demonstrates strong empirical performance even when this assumption is violated, as we show later in~\cref{sec:experiment}.
In practice, the drafting time can be hidden with appropriately designed pipeline~\citep{liu2024parallelspeculativedecodingadaptive,wang2024minions} which parallelizes the execution of the draft model and the target model.\footnote{Although, they have yet been integrated in popular battle-tested model serving frameworks such as \vllm{}~\citep{kwon2023vllm} and \sglang{}~\citep{zheng2024sglangefficientexecutionstructured} as of this writing.}
The true acceptance rates are inaccessible in practice, we thus rely on surrogate measures and show their empirical effectiveness, which we will discuss next.

\subsection{Practical Implementations}\label{sec:practical-implementations}

The acceptance rate of a draft token depends on $\max(p_\gM(d_i)/p_\gS(d_i),1)$. 
However, the \alg{} manager does not have access to $p_\gM(\cdot)$ before verification.
In practice, we use the draft model's output probability as a surrogate of the token acceptance rate~\citep{kim2023speculativedecodingbiglittle,zhang2024skiplayer}.
We show in~\cref{sec:experiment} that this surrogate empirically results in strong performance.
While prior works such as EAGLE-2~\citep{li2024eagle} and MDSD~\citep{hu2025mdsd} have adopted greedy token selection based on draft model probabilities on a single request, \alg{}' greedy algorithm operates at the
batch level to optimize resource utilization across multiple requests, where we defer a more detailed discussion to~\cref{app:related-work-greedy}.
Additionally, while we theoretically show that~\cref{alg:tetris} achieves a time complexity of $\gO(C\log N)$ (see~\cref{app:lem_time_comp}), we can additionally leverage the parallelism of GPU to achieve empirical negligible overhead of using \alg{} ($< 0.3 \text{ms}$ compared to the average draft time per token of $> 2.5\text{ms}$) via the \texttt{scatter\_max} operation directly implemented on GPU.
Lastly, the autoregressive token drafting can also be parallelized across requests.
Hence, drafting a batch of requests with a common window size of $k$ tokens takes the same time as a single request in practice.

    \section{Experiments}
    \label{sec:experiment}

We evaluate the effectiveness and efficiency of \alg{} against baseline methods.
We first validate the necessity of dynamic draft token selection and improvement of token acceptance with \alg{} in~\cref{sec:variation-draft-quality,sec:effect-extra-proposal}.
Then, we show the empirical end-to-end speedup in~\cref{sec:evaluation}.
We also discuss the potential further improvement in empirical results with the future implementation of speculative decoding pipelines in~\cref{sec:parallel-implementation}.
Our code is available at~\url{https://github.com/ZhaoxuanWu/Tetris}.

\begin{table}[b]
\centering
\caption{Server and model configurations. TP indicates the tensor parallel size used for model serving.}
\label{tab:server_config}
\resizebox{1\linewidth}{!}{
\begin{tabular}{@{}clll@{}}
\toprule
\textbf{Setting} & \textbf{Draft Model (TP)} & \textbf{Target Model (TP)} & \textbf{GPU (VRAM)} \\ 
\midrule
1 & Vicuna-68M (1) & Vicuna-33B (4) & 4$\times$L40 (180G) \\
2 & Llama-1B-FP8 (1) & Llama-70B (8) & 8$\times$L40 (360G) \\ 
3 &Llama-1B-FP8 (1) & Llama-405B-FP8 (8) & 8$\times$H100 (640G) \\ 
\bottomrule
\end{tabular}
}
\end{table}

\begin{figure*}[t]
    \centering
    \begin{minipage}{0.48\linewidth}
        \centering
        \includegraphics[width=0.9\linewidth]{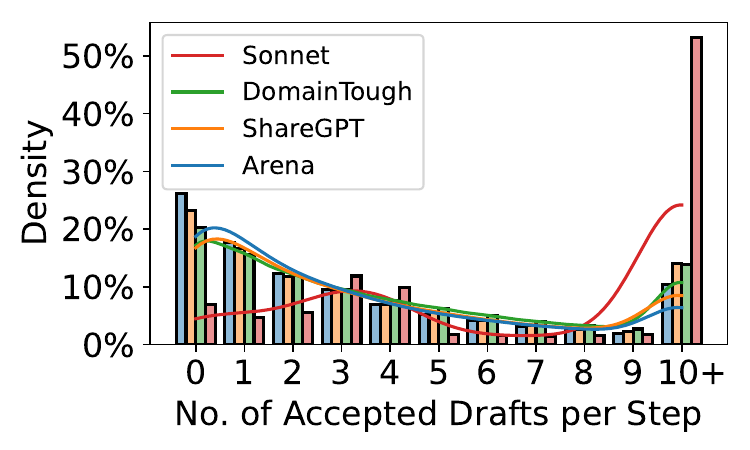}
        \vspace{-3mm} 
        \caption{The distribution of the number of accepted tokens per speculative decoding step for various tasks.}
        \label{fig:accepted_tokens}
    \end{minipage}\hfill
    \begin{minipage}{0.48\linewidth}
        \centering
        \includegraphics[width=0.9\linewidth]{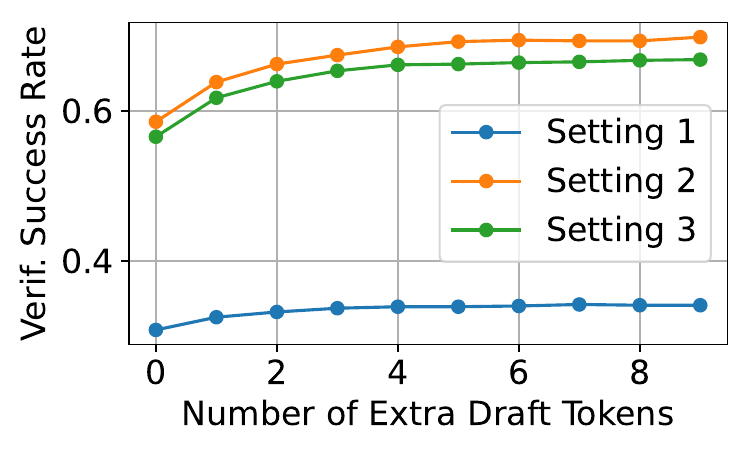}
        \vspace{-3mm} 
        \caption{Change in VSR as the number of \textit{extra draft tokens} increases. Base draft length $k$ is set to 4, results for other $k$'s are in~\cref{app:add-plots-extra-params}.}
        \label{fig:extra_params}
    \end{minipage}
\end{figure*}

\paragraph{Settings.}
We perform experiments on target models of various parameter sizes, including \textit{Vicuna-33B-v1.3}, \textit{Llama-3.1-70B-Instruct}, and \textit{Llama-3.1-405B-Instruct}.
We use \textit{Vicuna-68M} and \textit{Llama-3.2-1B-Instruct} as their respective draft models.
Depending on the size of the models, different server configurations and tensor parallel sizes are adopted, detailed in \cref{tab:server_config}.
\alg{} is evaluated for generation of answer completion for questions extracted from ShareGPT~\citep{sharegpt-dataset}, Chatbot Arena~\citep{zheng2023arena}, Domain Tough Questions~\citep{yav-ai2024domain-tough}, and synthetic tasks generated from Shakespeare's The Sonnet.
The standard speculative decoding (SD)~\citep{leviathan2023} and dynamic speculative decoding (DSD)~\citep{liu2024optimizingspeculativedecodingserving} are baseline methods that we compare to.
We vary the drafting window sizes, allowing up to 3 extra draft tokens for \alg{} while keeping the same number of tokens sent for verification by the target model (i.e., fixing the inference capacity) for fair comparison to baseline methods.
\alg{} is implemented in \vllm{}~\citep{kwon2023vllm}.


\subsection{Variations in Draft Quality}\label{sec:variation-draft-quality}

We begin by emphasizing the importance of setting an appropriate draft window size.
Using Setting 2, we collect the oracle optimal draft window size to adopt for each SD step.
Notably, the results in~\cref{fig:accepted_tokens} show flat curves with long-tail distributions for various datasets, revealing significant variations in optimal window size per step.
This diversity highlights the potential suboptimality of a fixed draft window, as it fails to adapt to the inherent characteristics of the draft-target model combination or a batch of sequences.
By tailoring the draft token selection in a batch, \alg{} is expected to achieve higher efficiency and better alignment with the model's token acceptance patterns, hence improving overall performance.

\subsection{Effect of Extra Draft Tokens}\label{sec:effect-extra-proposal}

Having extra draft tokens provides \alg{} with greater flexibility in selecting which draft tokens to send for verification.
To empirically show this effect, we define the verification success rate (VSR), 
\begin{equation}\label{eq:VSR}
   \textstyle \textit{VSR} = \frac{\textit{Accepted tokens}}{\textit{Tokens sent for verification}} \ ,
\end{equation}
which measures the quality of the draft tokens selected by \alg{}.
Fixing the total number of tokens sent for verification, we show in~\cref{fig:extra_params} that increasing the number of extra draft tokens consistently increases the VSR metric across all settings.
This finding confirms the effectiveness of \alg{}'s strategy for draft token selection utilizing extra draft tokens.
It also validates the empirical usefulness of the draft model's output probabilities as a surrogate of the selection criteria, as stated in~\cref{sec:practical-implementations}.

\subsection{Evaluation of \alg}\label{sec:evaluation}

\begin{figure*}[t]
    \centering
    \setlength{\tabcolsep}{1pt} 
    \resizebox{0.99\linewidth}{!}{
    \begin{tabular}{cccc}
        & \hspace{8mm}\textbf{ShareGPT} & \hspace{8mm}\textbf{Arena} & \hspace{8mm}\textbf{Tough} \\
        \rotatebox{90}{\parbox{2.5cm}{\centering \hspace{8mm}\textbf{Setting 1}}} & \includegraphics[width=0.31\linewidth]{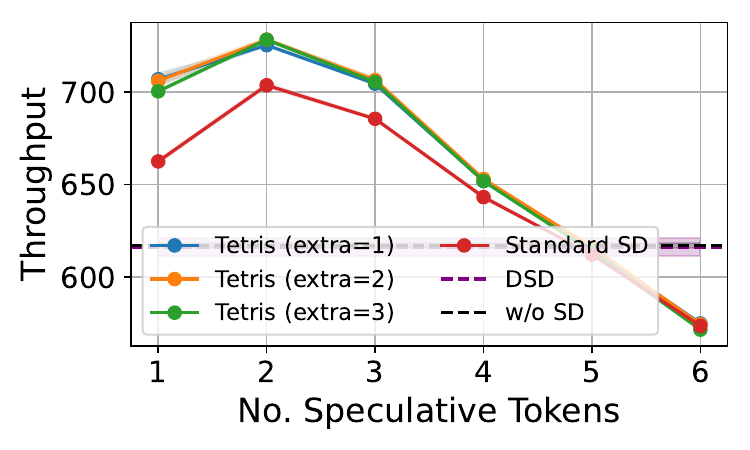} &
        \includegraphics[width=0.31\linewidth]{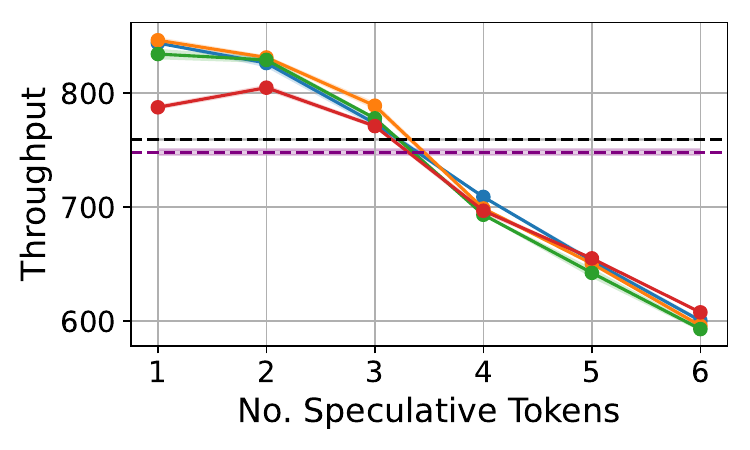} &
        \includegraphics[width=0.31\linewidth]{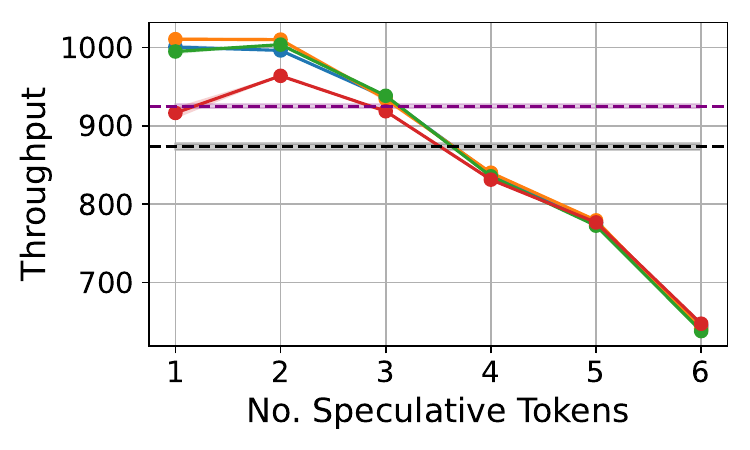} \\
        & (a) \hspace{5mm} $\uparrow$ 3.50\%, $\Delta$ 6.70\% & (b) \hspace{5mm} $\uparrow$ 5.17\%, $\Delta$ 7.47\% & (c) \hspace{5mm} $\uparrow$ 4.85\%, $\Delta$ 9.27\% \\
        
        \rotatebox{90}{\parbox{2.5cm}{\centering \hspace{8mm}\textbf{Setting 2}}} & \includegraphics[width=0.31\linewidth]{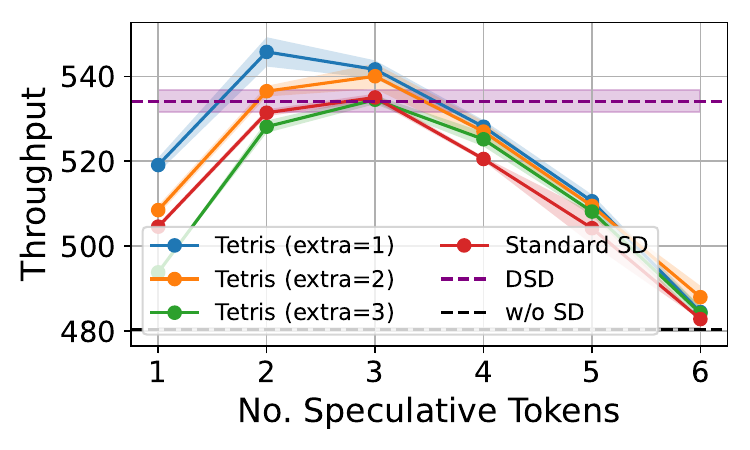} & \includegraphics[width=0.31\linewidth]{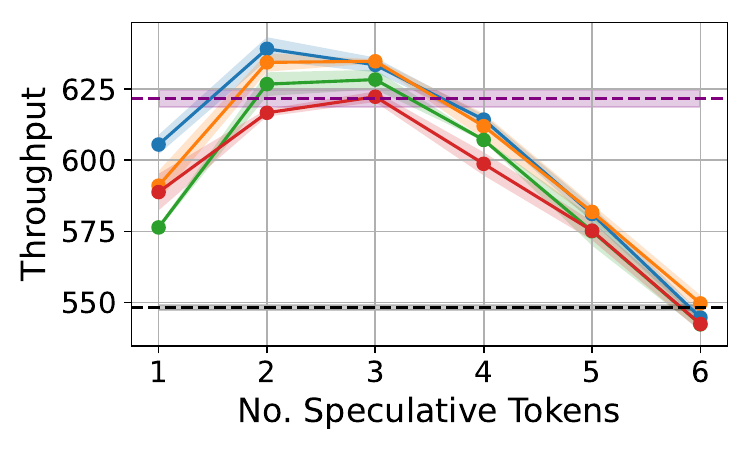} & \includegraphics[width=0.31\linewidth]{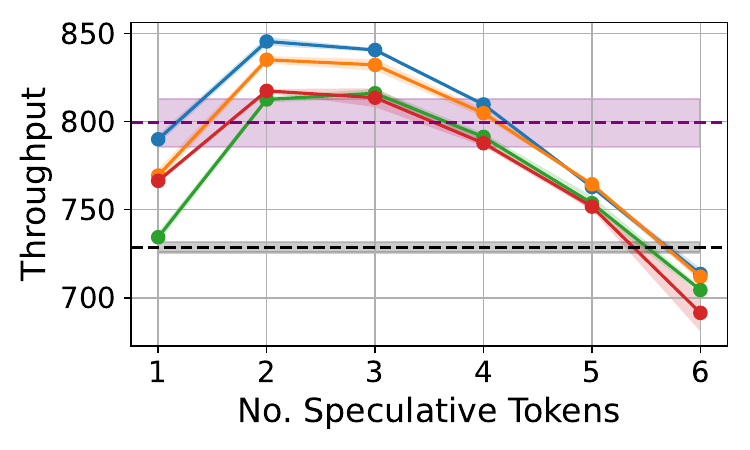}  \\
        & (d) \hspace{5mm} $\uparrow$ 2.01\%, $\Delta$ 2.17\% & (e) \hspace{5mm} $\uparrow$ 2.71\%, $\Delta$ 2.81\% & (f) \hspace{5mm} $\uparrow$ 3.43\%, $\Delta$ 3.43\% \\
        
        \rotatebox{90}{\parbox{2.5cm}{\centering \hspace{8mm}\textbf{Setting 3}}} & \includegraphics[width=0.31\linewidth]{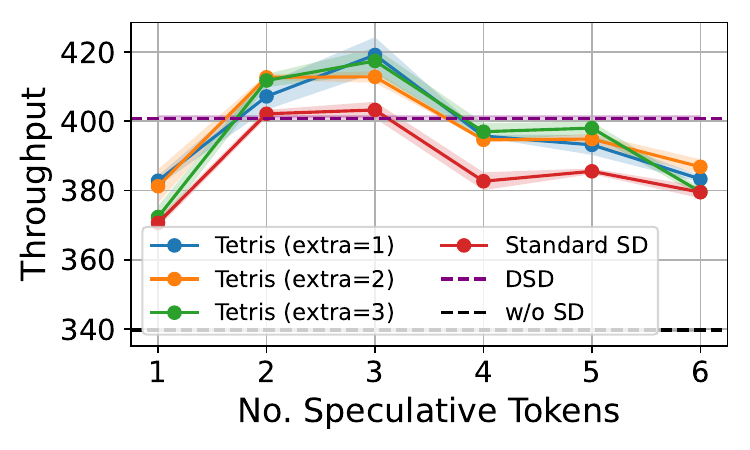} & 
        \includegraphics[width=0.31\linewidth]{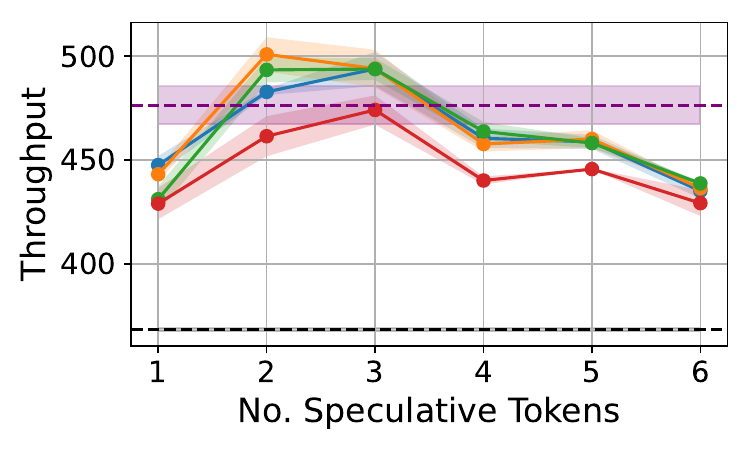} &
        \includegraphics[width=0.31\linewidth]{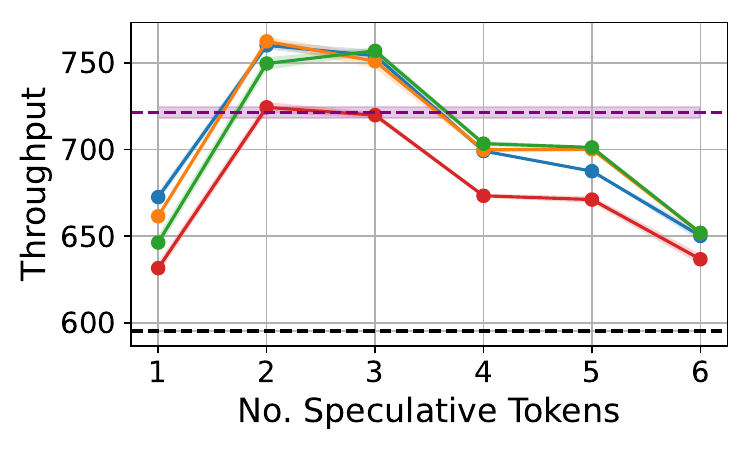} \\
        & (g) \hspace{5mm} $\uparrow$ 3.93\%, $\Delta$ 3.93\% & (h) \hspace{5mm} $\uparrow$ 5.15\%, $\Delta$ 5.15\% & (i) \hspace{5mm} $\uparrow$ 5.25\%, $\Delta$ 5.25\% \\
    \end{tabular}}
    \caption{Throughput comparison for various methods across experimental settings. $\uparrow$ indicates the improvement over the best baseline method. $\Delta$ indicates the maximum gap between \alg{} and standard SD. The reported numbers reflect the mean and standard deviation over 3 independent trials.}
    \label{fig:throughput-all}
\end{figure*}

To evaluate the effectiveness of \alg{}, we perform comprehensive experiments on various datasets and report metrics, including the total throughput and end-to-end latency.
We compare to standard SD and DSD.
Throughout the experiments, we maintain a consistent system load of 64 batched requests to ensure consistency, reproducibility, and fairness in comparisons. Note that all experiments include drafting time.

\paragraph{Total Throughput.}
We measure the performance of a speculative decoding method using the total throughput, which includes both accepted draft tokens by the target model and the bonus tokens, which make up the final completion.
As shown in \cref{fig:throughput-all}, \alg{} achieves up to approximately 5.25\% improvement in terms of total throughput compared to the best baseline, depending on the draft-target setting and the nature of the task performed.
The maximum gap between \alg{} and standard SD is up to 9.27\%.
Importantly, \alg{} consistently outperforms the standard SD and DSD across all settings of the draft window sizes.
This shows the robustness of \alg{} to different hyperparameter choices.
Additionally, it is evident that having more speculative tokens (i.e., a larger draft window size) does not always improve the performance, as having too many parallel executions of the target model exceeding the servers' parallel inference capacity degrades performance.

Empirically, we observe that \alg{} achieves optimal performance when the number of extra draft tokens is set to 1 or 2.
These results are partly attributed to the current sequential draft-target implementation for the speculative decoding pipeline, as more extra draft tokens take time to generate autoregressively.
Remarkably, this pipeline can be better designed to amplify the benefit of \alg{}, which we defer the discussion to \cref{sec:parallel-implementation}.
Moreover, while DSD is expected to outperform standard SD, we note that it is not always the case in empirical experiments.
This behavior may result from the difficulty of accurately estimating the conditional token acceptance rate in practice\footnote{Inaccurate conditional acceptance rate estimation results in inaccurate calculation of expected generation token counts.} and the quality of the fitted latency prediction model.

\paragraph{End-to-end Latency.}
We also measure the end-to-end latency of each request, defined as the time from sending the request to receiving the final response from the \vllm{} server on the client side.
This metric measures the average latency of the speculative decoding system in finishing completions, which can affect user satisfaction.
We summarize the results in~\cref{tab:latency-all} and defer the figures to~\cref{app:latency-plots}.
Overall, \alg{} achieves up to 6.13\% improvement in latency as compared to the best baseline and up to 9.32\% improvement against standard SD.

\begin{table}[t]
    \centering
    \caption{Improvement in end-to-end latency. Refer to~\cref{fig:throughput-all} for definitions of $\uparrow$ and $\Delta$. The reported numbers reflect the mean over 3 independent trials.}
    \label{tab:latency-all}
    \resizebox{0.99\linewidth}{!}{
    \begin{tabular}{c|cc|cc|cc}
    \toprule
    \multirow{2}{*}{Setting} & \multicolumn{2}{c|}{\textbf{ShareGPT}} & \multicolumn{2}{c|}{\textbf{Arena}} & \multicolumn{2}{c}{\textbf{Tough}} \\ 
    & $\uparrow$ & $\Delta$ & $\uparrow$ & $\Delta$ & $\uparrow$ & $\Delta$ \\
    \midrule
    \multirow{1}{*}{1} & 3.42\% & 6.05\% & 5.30\% & 6.30\% & 5.47\% & 9.32\% \\
    \midrule
    \multirow{1}{*}{2} & 2.65\% & 2.70\% & 3.86\% & 3.86\% & 3.65\% & 3.65\% \\
    \midrule
    \multirow{1}{*}{3} & 3.51\% & 4.52\% & 6.13\% & 6.13\% & 4.49\% & 4.68\% \\
    \bottomrule
    \end{tabular}}
\end{table}

\begin{figure*}[t]
    \centering
    \includegraphics[width=0.6\linewidth]{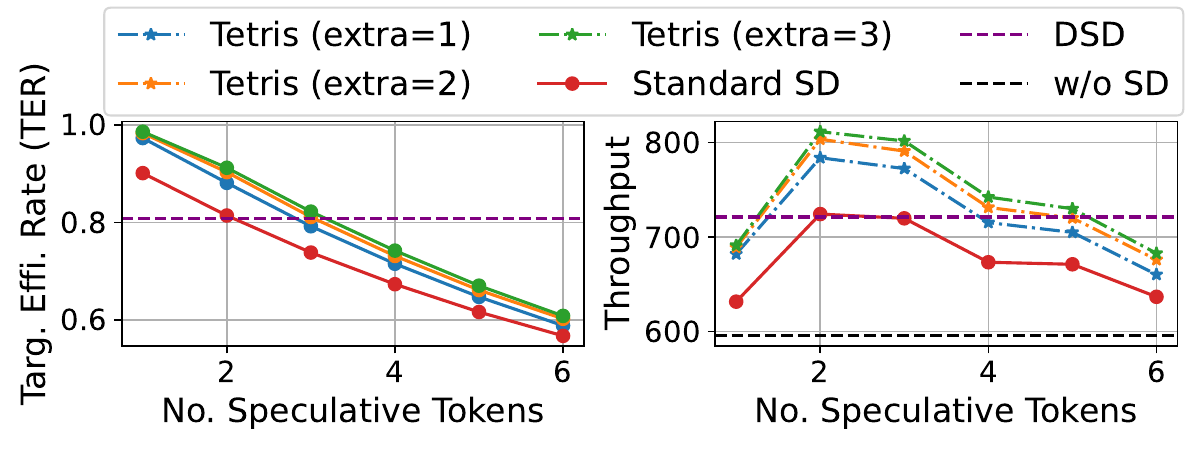}
    \caption{Left: Baseline comparisons for \textit{TER} in different speculative configurations. Right: \textit{Projected} $\hat{\gG}^{(\textit{TER})}_k$ plot for \alg{} with baselines.}
    \label{fig:TER-figure}
\end{figure*}

\begin{figure}[t]
    \centering
    \includegraphics[width=0.99\linewidth]{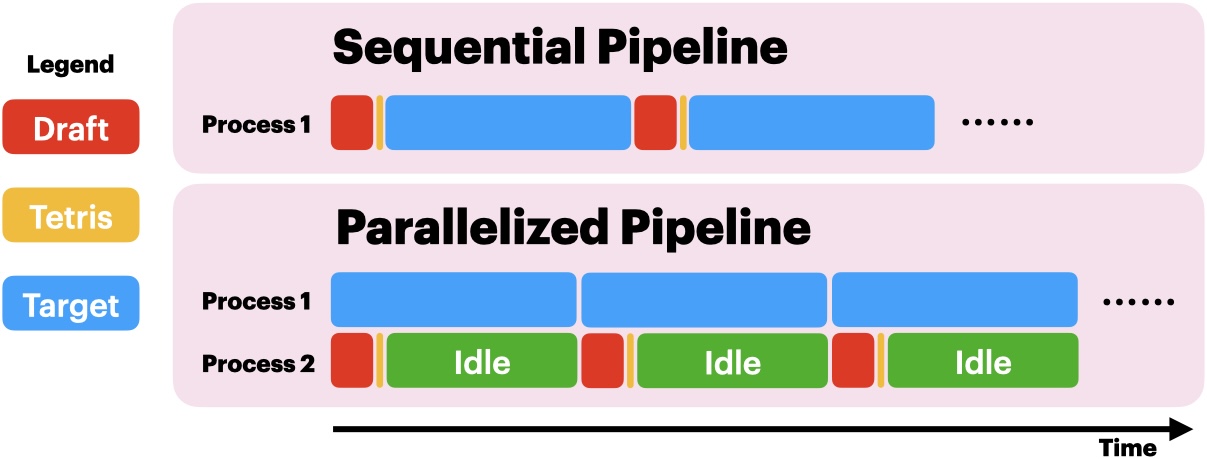}
    \caption{Parallelized pipeline for speculative decoding, where the draft model and \alg{} runtime can be hidden entirely through parallelization.}
    \label{fig:pipeline}
\end{figure}

\subsection{Potentially Parallelized Pipeline}\label{sec:parallel-implementation}
We implement \alg{} to work with the \vllm{} library, one of the most efficient frameworks for LLM inference~\citep{kwon2023vllm}.
\vllm{} adopts a sequential pipeline for speculative decoding, where the target model runs sequentially after the draft model finishes generating draft tokens.
As illustrated in~\cref{fig:pipeline}, \alg{} is integrated between the draft and target models.
However, in such a sequential pipeline, \alg{} cannot fully realize its potential as the extra draft tokens incur additional computational time.

Recent works such as Minions~\citep{wang2024minions} and PEARL~\citep{liu2024parallelspeculativedecodingadaptive} have started exploring the benefits of a parallelized pipeline with two processes concurrently running the draft and target models as illustrated in~\cref{fig:pipeline}.
Given that the draft model runs significantly faster than the target model, the draft time, as well as the time to run our \alg{}, can be hidden entirely in the parallelized pipeline.
Moreover, the idle time (marked in green) of Process 2 between steps can be utilized to draft more extra tokens of \alg{} or to run more complex algorithms.

Under the constraint of sequential pipelines in \vllm{}, we instead adopt an alternative performance metric that better captures the potential advantages of \alg{} in parallelized pipelines.
We use the target efficiency rate (TER) defined as follows,
\begin{equation}\label{eq:TER}
  \textstyle  \textit{TER} = \frac{\textit{Accepted tokens + Bonus tokens}}{\textit{Max possible number of tokens if all accepted}} \ .
\end{equation}
As \textit{TER} measures the efficiency of target model verifications and is unaffected by the drafting process and \alg{} runtime, it provides an accurate indication of the net benefit of \alg{}.
In~\cref{fig:TER-figure}, we demonstrate a case study for Setting 3 on Tough dataset: The improvement of \textit{TER} is first calculated from the left figure, and is then used to compute the \textit{projected throughput} $\hat{\gG}^{(\textit{TER})}_k$, following
\begin{equation*}
     \textstyle \hat{\gG}^{(\textit{TER})}_k = \gG_{\text{SD}, k} \times \frac{ (\textit{TER}_{\alg, k} - \textit{TER}_{\text{SD},k})} {\textit{TER}_{\text{SD},k}} \ ,
\end{equation*}
where $k$ is the number of speculative tokens (i.e., draft window size) and $\gG$ represents throughput.
Consequently, using \alg{} is \textit{projected} (i.e., not realized in the current implementation) to achieve 12.04\% improvement for this setting under parallelized pipeline.
The full results are shown in~\cref{tab:potential} and the figures are shown in~\cref{app:projected-improvement-plots}.

\begin{table}[t]
    \centering
    \caption{Projected throughput $\hat{\gG}^{(\textit{TER})}$ improvement based on \textit{TER} metric improvement, realizable under a parallelized speculative decoding pipeline.}
    \label{tab:potential}
    \resizebox{0.7\linewidth}{!}{
        \begin{tabular}{@{}clcc@{}}
            \toprule
            \textbf{Setting} & \textbf{Dataset} & \textbf{$\gG$$\uparrow$} & \textbf{$\hat{\gG}^{(\textit{TER})}$$\uparrow$} \\ 
            \midrule
            \multirow{3}{*}{1} 
            & ShareGPT & 3.50\% & 9.70\% \\
            & Arena & 5.17\% & 7.79\% \\
            & Tough & 4.85\% & 8.92\% \\
            \midrule
            \multirow{3}{*}{2} 
            & ShareGPT & 2.01\% & 11.70\% \\
            & Arena & 2.71\% & 11.17\% \\
            & Tough & 3.43\% & 11.91\% \\
            \midrule
            \multirow{3}{*}{3} 
            & ShareGPT & 3.93\% & 11.67\% \\
            & Arena & 5.15\% & 10.53\% \\
            & Tough & 5.25\% & 12.04\% \\
            \bottomrule
        \end{tabular}
    }
\end{table}

\subsection{Ablation Study}\label{sec:ablation-study}

\paragraph{Robustness to Variations in Draft Quality.}
We artificially introduce additional variations in draft quality by mixing datasets of different difficulty levels.
We create synthetic prompts designed for models to repeat lines from a poem named Sonnet.
Since Sonnet is relatively easy for the small draft model, it achieves a high rate of successful verification by the target model.
We then construct a new dataset, Mix, by randomly mixing Sonnet and a more challenging dataset, Tough, in equal proportions.
As shown~\cref{tab:ablation-mix}, the performance improvement of \alg{} over the best baseline suffers only a marginal or no decline, indicating its robustness to substantial variations in draft quality.
\begin{table}[t]
    \centering
    \caption{\alg{} improvement in throughput for ablation study of robustness to variations in draft quality.}
    \label{tab:ablation-mix}
    \resizebox{0.7\linewidth}{!}{
        \begin{tabular}{cccc}
            \toprule
            \multirow{1}{*}{Setting} & \multicolumn{1}{c}{\textbf{Sonnet}} & \multicolumn{1}{c}{\textbf{Tough}} & \multicolumn{1}{c}{\textbf{Mix}}  \\ 
            \midrule
            \multirow{1}{*}{1} & 2.46\% &  4.85\% &  4.12\% \\
            \multirow{1}{*}{2} & -0.81\% & 3.43\% & 3.48\% \\
            \multirow{1}{*}{3} & 2.07\% &  5.25\% & 4.24\% \\
            \bottomrule
        \end{tabular}
    } 
\end{table}

\paragraph{Extension to Medusa.}
The Medusa model generates multiple subsequent draft tokens using a single forward pass (as opposed to autoregressive generation) through multiple decoding heads~\citep{cai2024medusa}.
Leveraging Medusa, it is possible to generate extra draft tokens for \alg{} at minimal marginal computational cost.
We show in~\cref{app:medusa-extension} that integrating \alg{} to Medusa achieves a 3.19\% improvement in total throughput.

\paragraph{Other Ablations.}
We also include ablations on \alg{}'s improvement in verification success rate (VSR) in~\cref{app:improve-VSR}, and the effect of batch size on the performance in~\cref{app:batch-size}.

    \section{Conclusion and Future Work}
    \label{sec:}

In this paper, we study the problem of optimizing batch speculative decoding to maximize throughput in multi-request settings, such as those faced by model service providers.
To this end, we propose \alg{}, a novel method that efficiently selects optimal draft tokens for the LLM verification in log-linear time.
We have theoretically shown that, in the absence of drafting time, \alg{} achieves optimal throughput both at each decoding step and globally under reasonable assumptions about token acceptance rates.
Our empirical results further validate that \alg{} consistently outperforms standard speculative decoding and existing dynamic draft window selection methods, even when accounting for the extra time required for drafting extra tokens. 
These results highlight the potential of \alg{} to improve inference efficiency in real-world model service deployments.
A key future direction is adapting \alg{} to tree decoding, a notable feature in recent advancements in speculative decoding~\citep{cai2024medusa,li2024eagle2,li2024eagle}. Another interesting direction to explore is to design better draft token selection techniques. Insights can be derived from existing works on data selection~\citep{lin2024distributionally,zhou2022probablyapproximateshapleyfairness,xu2024datacentricaiagelarge,zhou2024detail} that find important data points. Similarly, further research can investigate how to leverage the probability distribution between the draft model and the target model to improve the selection efficiency.

    \section{Limitations}

In this paper, our empirical experiments only demonstrate results using the current sequential speculative decoding pipeline implemented on \vllm. 
That is, the target model stays idle while waiting for draft tokens from the draft model. 
Consequently, the performance improvement of \alg{} is heavily dependent on the trade-off between the additional runtime required to generate extra draft tokens and the gain in token acceptance achieved through \alg{}.
Such trade-off limits the practical effectiveness of \alg{}, especially when a slow draft model is required.
We anticipate that future implementations of a parallelized pipeline could potentially reveal greater speedups with \alg{}.
However, we have not yet integrated such features into \vllm{} for testing in empirical experiments.

    \section*{Acknowledgments}
    This research is supported by the National Research Foundation (NRF), Prime Minister's Office, Singapore under its Campus for Research Excellence and Technological Enterprise (CREATE) programme. The Mens, Manus, and Machina (M3S) is an interdisciplinary research group (IRG) of the Singapore MIT Alliance for Research and Technology (SMART) centre.
    
    \bibliography{references}

    \appendix
    \onecolumn
    \label{sec:appendix}

\section{Leftover Proofs}\label{app:proofs}

\subsection{Proof of \texorpdfstring{\cref{thm:tetris}}{Theorem 1}.}
\label{app:proof_of_local_tetris}

\tetris*
\begin{proof}
    We prove it by contradiction. Let the selection of \cref{alg:tetris} be $\gD^*$. Suppose the actual optimal solution is $\gD' \neq \gD^*$. Let $\tilde{\gD} = \gD' \cap \gD^*$ be the overlapping tokens selected by both \cref{alg:tetris} and the actual optimal solution. Note that the tokens in each row are selected sequentially (i.e., tokens cannot be skipped in a row). 

    \paragraph{Case 1: \alg{} selects some token $d \in \gD^* \setminus \tilde{\gD}$ before selecting $\tilde{\gD}$.} In this case, the $\mathbb{E}[\vone]$ of the token $d$ is higher than the token last selected in $\tilde{\gD}$. This suggests that the optimal selection should include $d$. However, it can be observed that $d \notin D'$ since otherwise $d \in \tilde{\gD}$. This contradicts the fact that $\gD'$ is optimal.

    \paragraph{Case 2: \alg{} selects $\tilde{\gD}$ first before selecting other tokens.} Since \cref{alg:tetris} always selects the token with the highest $\mathbb{E}[\vone]$, every element in $\gD^* \setminus \tilde{\gD}$ is larger than or equal to that in $\gD' \setminus \tilde{\gD}$. As such, we have 
    $\mathbb{E}[\sum_{p \in \gD'} \mathbf{1}_p] \leq \mathbb{E}[\sum_{p \in \gD^*} \mathbf{1}_p]$. However, this contradicts the fact that $\gD'$ is optimal as \cref{alg:tetris} has a higher number of accepted tokens. Therefore, \cref{alg:tetris} must be optimal. 
    
    Combining the two cases finishes the proof. \qedhere
\end{proof}

\subsection{Running Time of \alg{}}
\label{app:lem_time_comp}

\begin{lem}
\label{lem:time_comp}
\cref{alg:tetris} achieves a time complexity of $\gO(C\log N)$.

\begin{proof}
    Note that \cref{alg:tetris} maintains a heap. The heap is initialized with $N$ items. Since only $C$ pairs are selected, there are $2C$ operations of \texttt{enqueue} and \texttt{dequeue}. Following classic results of heap operation, each \texttt{enqueue} of \texttt{dequeue} operation requires $\gO(\log C)$ time. As such, the overall time complexity of \alg{} is $\gO(C\log N)$.
\end{proof}
\end{lem}

\subsection{Proof of \texorpdfstring{\cref{thm:globalOpt}}{Theorem 2}.}
\label{app:thm_global_opt_tetris}

\globalOpt*
\begin{proof}
    The proof of global optimality is established on \cref{thm:tetris}. Since all tokens in each row have the same acceptance rate. After each step, we have the same distribution of $\vone$ no matter what tokens are accepted, where $\vone$ is the indicator variable of whether the token is accepted.
    As such, at each step, performing \alg{} is per-step optimal by \cref{thm:tetris}. Moreover, since the state at each step is identical, a per-step optimal strategy is also globally optimal. \qedhere
\end{proof}

\section{Additional Related Work and Discussion}

\subsection{Acceptance Rate}\label{app:acceptance-related-work}

The acceptance rate plays a vital role in the effectiveness of speculative decoding.
A higher acceptance rate should be paired with a larger draft window size $k$ to achieve optimal speedup.
In the typical rejection sampling setting of speculative decoding, the acceptance of draft tokens depends on the probability distributions of both the draft and target models.
When the probability distribution of the draft model, $p_\gS(\cdot)$, closely approximates that of the target model, $p_\gM(\cdot)$, a higher number of tokens are accepted on average.
Since the value of $k$ is chosen in the drafting process, we do not have access to $p_\gM(\cdot)$ and have to rely on $p_\gS(\cdot)$ to estimate the acceptance rate.

\citet{leviathan2023} derive that the expected acceptance rate is 1 minus the KL divergence between the token distributions of the draft and the target model.
Hence, the acceptance rates for all draft tokens are considered constant.
\citet{liu2024optimizingspeculativedecodingserving} assume uniform token acceptance behavior across diverse requests.
It proposes SmartSpec, which calculates the average acceptance rate from past generation steps.
\citet{li2024eagle2} and~\citet{wang2024opttreespeculativedecodingadaptive} utilize the draft model's confidence score (i.e., the output probability of each token) to estimate the acceptance rate.
\citet{chen2024sequoia} make the positional acceptance assumption so that the acceptance rate of tokens is determined solely by their position (i.e., number of tokens away) relative to the already accepted tokens.
\citet{agrawal2024adaedlearlydraftstopping} instead consider an approximate lower bound on the expected acceptance rate of a token that depends on the entropy of prediction probabilities of the draft model $p_\gS(\cdot)$.
Noting the acceptability of diverse tokens, especially in the real world with a high value of temperature hyperparameter, Medusa proposes to use both a hard threshold and an entropy-dependent threshold as a criterion to accept draft tokens~\citep{cai2024medusa}.
In Medusa, the first token is always accepted using greedy decoding to ensure at least one token is generated in each step.

\subsection{Greedy Algorithms for Speculative Decoding}\label{app:related-work-greedy}

Several existing approaches in speculative decoding employ greedy algorithms to improve efficiency, though they operate at different levels and with distinct objectives. EAGLE-2~\citep{li2024eagle2} and MDSD~\citep{hu2025mdsd}, for instance, utilize greedy selection strategies within single-request, multi-draft settings. EAGLE-2 achieves speedups by generating multiple branches in a draft tree, which the top candidates are then greedily selected for verification. It requires substantial GPU resources for each request and may not translate to improved throughput in batch inference due to limited total GPU capacity. MDSD, on the other hand, focuses on increasing the acceptance rate in multi-draft speculative decoding through greedy draft sampling, an objective orthogonal to our work.
Our \alg{} distinguishes itself by applying a greedy algorithm at the batch level, optimizing GPU resource utilization across multiple requests rather than within a single draft tree. Due to the simplicity of greedy approaches in implementation, our method can be readily adapted to major inference frameworks (e.g., \vllm{}~\citep{kwon2023vllm}) with minimal empirical overhead. Our paper also demonstrates that \alg{} is complementary to existing speculative decoding frameworks, such as Medusa~\citep{cai2024medusa} (see~\cref{sec:ablation-study} and~\cref{app:medusa-extension}), by introducing a novel axis of optimization not explored by prior works.

\subsection{Assumptions for Analysis}\label{app:related-work-assumption}

Our theoretical analysis of the global optimality of \alg{} in~\cref{thm:globalOpt} is conditioned on~\cref{assu:equal_acc}.
Similar assumptions are commonly made in the literature~\citep{leviathan2023,liu2024optimizingspeculativedecodingserving} to analyze the length of generated outputs. 
Specifically, the pioneering work of \citet{leviathan2023} assumes a constant acceptance rate $\alpha$ and applies a capped geometric distribution to derive the expected number of generated tokens, under a single-request setting.
For a single request, this assumption is equivalent to our~\cref{assu:equal_acc}.
A natural extension of the constant acceptance rate assumption to the batch inference setting would be assuming the identical acceptance rate across all requests. 
However, we highlight that our~\cref{assu:equal_acc} is weaker than assuming the same constant acceptance rate across all requests, as it allows the acceptance rate to vary across requests.
Moreover, even when~\cref{assu:equal_acc} is violated, our method still demonstrates strong empirical performance as shown in~\cref{sec:experiment} of the main paper.

\begin{figure}[b]
    \centering
    \setlength{\tabcolsep}{1pt} 
    \begin{tabular}{ccc}
        \includegraphics[width=0.31\linewidth]{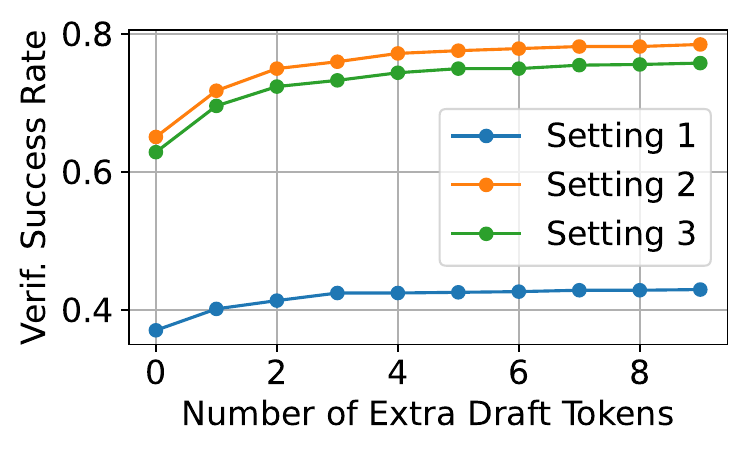} &
        \includegraphics[width=0.31\linewidth]{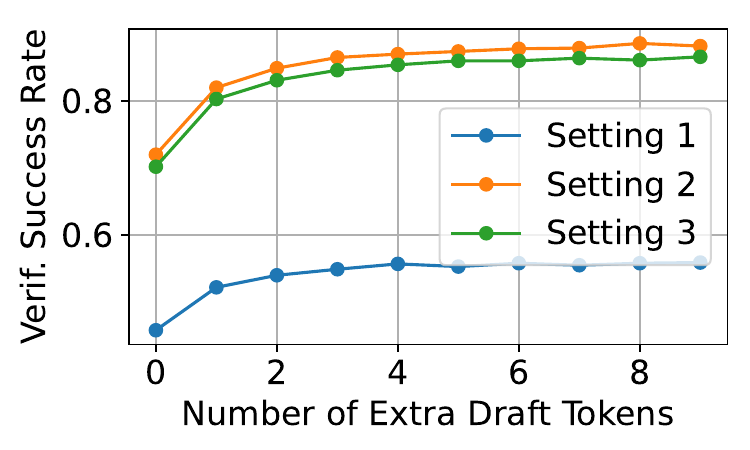} &
        \includegraphics[width=0.31\linewidth]{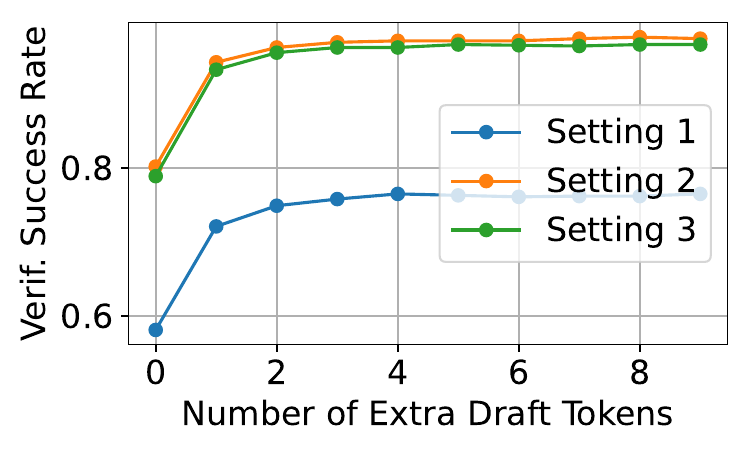} \\
        (a) Base $k=3$ & (b) Base $k=2$ & (c) Base $k=1$ \\
    \end{tabular}
    \caption{Additional plots for the change in VSR as the number of \textit{extra draft tokens} increases, supplementary to~\cref{fig:extra_params} where base draft length $k=4$. Here, we provide results for $k=1,2,3$.}
    \label{fig:extra_params_additional_k123}
\end{figure}

\section{Additional Results}

\subsection{Dataset License}\label{app:data-license}
ShareGPT~\citep{sharegpt-dataset}: Apache license 2.0; Arena~\citep{zheng2023arena}: CC; Domain-specific Tough Questions~\citep{yav-ai2024domain-tough}: MIT.

\subsection{Additional Plots for Effect of Extra Draft Tokens}\label{app:add-plots-extra-params}

\cref{fig:extra_params} demonstrates the benefit of adding extra draft tokens (from 0 to 9 extra tokens).
The base draft length $k$ serves only as a reference baseline and its value does not significantly impact the observed trends in~\cref{fig:extra_params}.
We set $k=4$ based on established reasonable ranges used in existing studies~\citep{leviathan2023,zhang2024skiplayer,wang2024minions}.
This choice balances the trade-off between extra compute required and additional throughput gained.

For completeness, we conducted additional experiments with $k=1,2,3$ and present the results in~\cref{fig:extra_params_additional_k123}. 
The observed trends are consistent across different values of $k$: VSR generally increases with extra tokens, demonstrating the effectiveness of having extra draft tokens in our \alg{}.

\begin{figure}[t]
    \centering
    \setlength{\tabcolsep}{1pt} 
    \begin{tabular}{cccc}
    & \hspace{8mm}\textbf{ShareGPT} & \hspace{8mm}\textbf{Arena} & \hspace{8mm}\textbf{Tough} \\
    \rotatebox{90}{\parbox{2.5cm}{\centering \hspace{8mm}\textbf{Setting 1}}} & \includegraphics[width=0.31\linewidth]{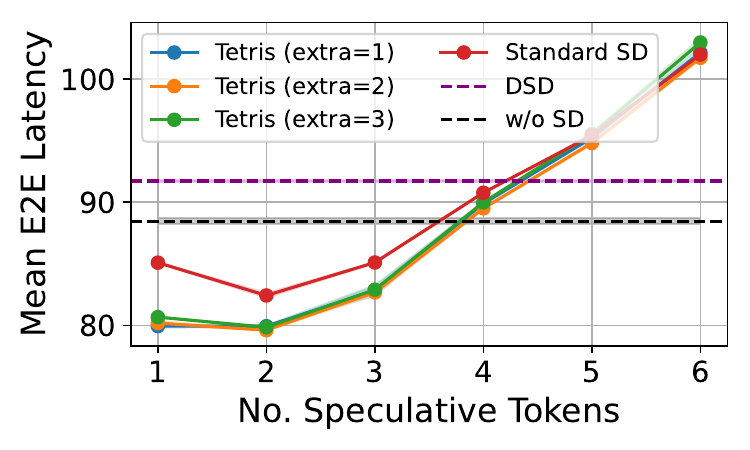} &
    \includegraphics[width=0.31\linewidth]{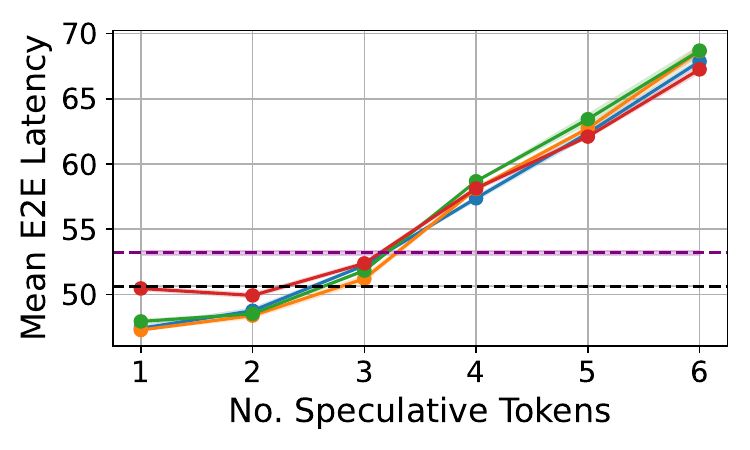} &
    \includegraphics[width=0.31\linewidth]{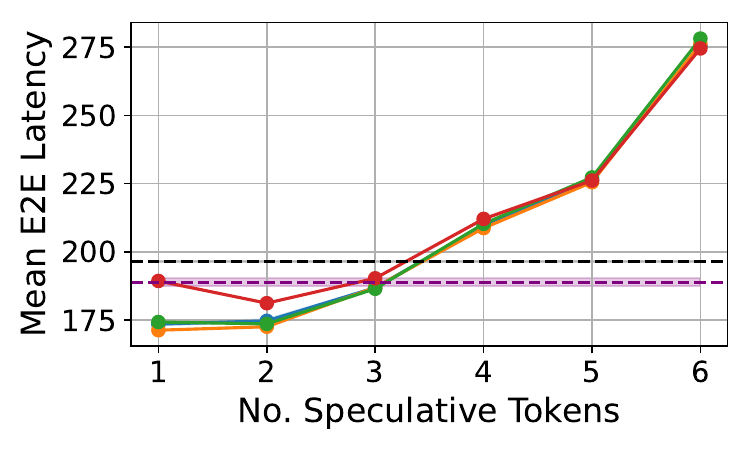} \\
    & (a) \hspace{5mm} $\uparrow$ 3.42\%, $\Delta$ 6.05\% & (b) \hspace{5mm} $\uparrow$ 5.30\%, $\Delta$ 6.30\% & (c) \hspace{5mm} $\uparrow$ 5.47\%, $\Delta$ 9.32\% \\
    
    \rotatebox{90}{\parbox{2.5cm}{\centering \hspace{8mm}\textbf{Setting 2}}}  & \includegraphics[width=0.31\linewidth]{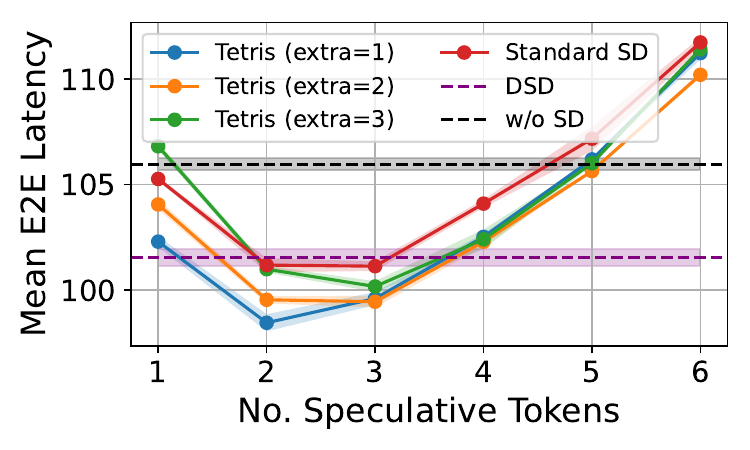} & \includegraphics[width=0.31\linewidth]{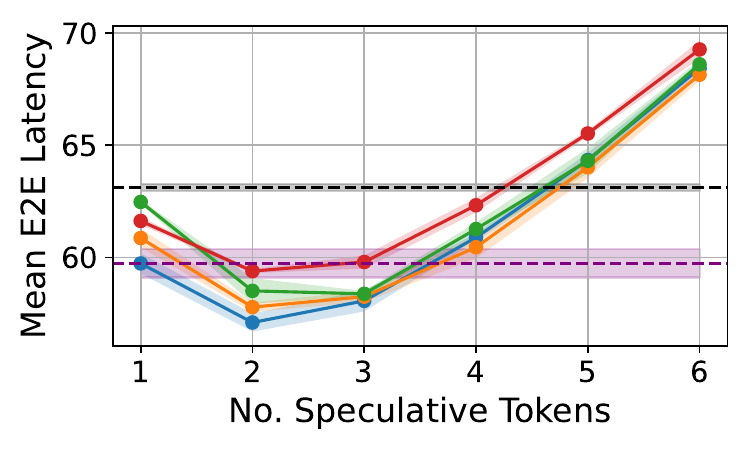} & \includegraphics[width=0.31\linewidth]{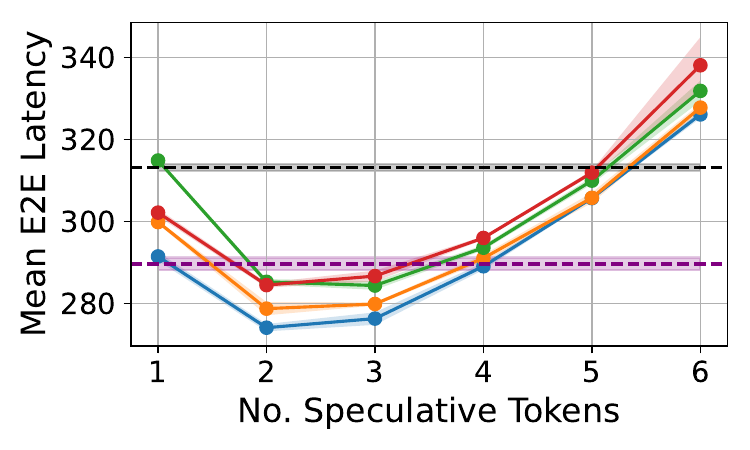}  \\
    & (d) \hspace{5mm} $\uparrow$ 2.65\%, $\Delta$ 2.70\% & (e) \hspace{5mm} $\uparrow$ 3.86\%, $\Delta$ 3.86\% & (f) \hspace{5mm} $\uparrow$ 3.65\%, $\Delta$ 3.65\% \\
    
    \rotatebox{90}{\parbox{2.5cm}{\centering \hspace{8mm}\textbf{Setting 3}}}  & \includegraphics[width=0.31\linewidth]{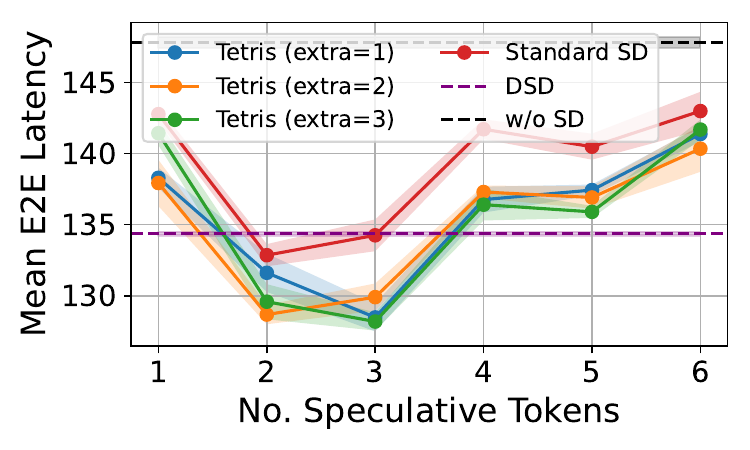} & 
    \includegraphics[width=0.31\linewidth]{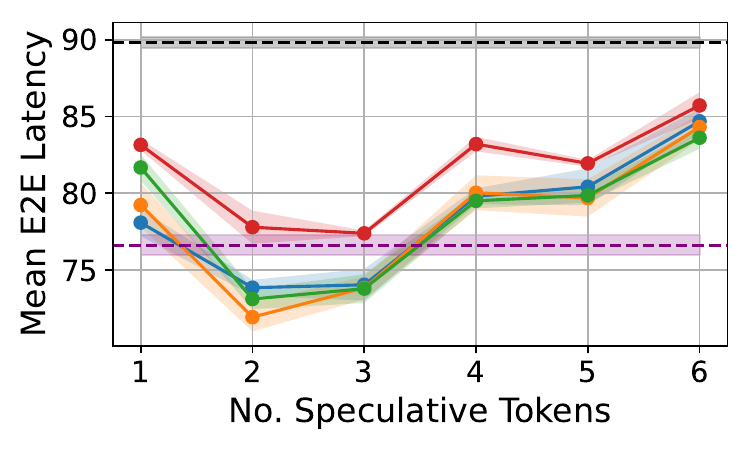} &
    \includegraphics[width=0.31\linewidth]{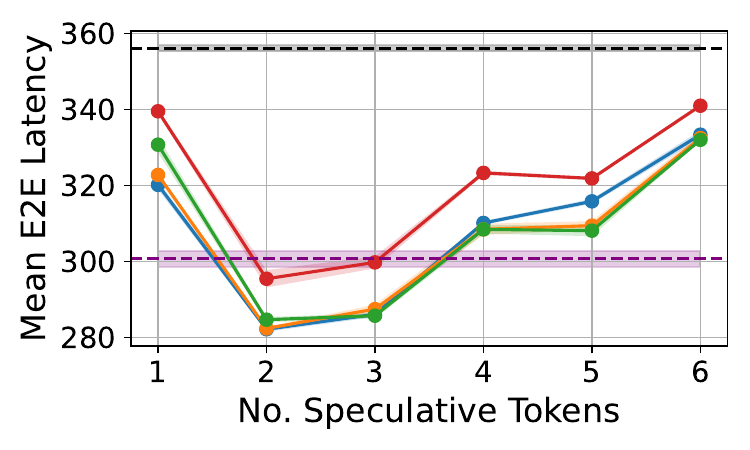} \\
    & (g) \hspace{5mm} $\uparrow$ 3.51\%, $\Delta$ 4.52\% & (h) \hspace{5mm} $\uparrow$ 6.13\%, $\Delta$ 6.13\% & (i) \hspace{5mm} $\uparrow$ 4.49\%, $\Delta$ 4.68\% \\
\end{tabular}
    \caption{Mean end-to-end latency comparison for various methods across experimental settings. $\uparrow$ indicates the improvement from best baseline method. $\Delta$ indicates the maximum gap between \alg{} and standard SD. The reported numbers reflect the mean and standard deviation over 3 independent trials.}
    \label{fig:latency-all-appendix}
\end{figure}

\subsection{Plots for End-to-end Latency}\label{app:latency-plots}

\begin{figure}[t]
    \centering
    \setlength{\tabcolsep}{1pt} 
    \begin{tabular}{cccc}
    & \hspace{8mm}\textbf{ShareGPT} & \hspace{8mm}\textbf{Arena} & \hspace{8mm}\textbf{Tough} \\
    \rotatebox{90}{\parbox{2.5cm}{\centering \hspace{8mm}\textbf{Setting 1}}} & \includegraphics[width=0.31\linewidth]{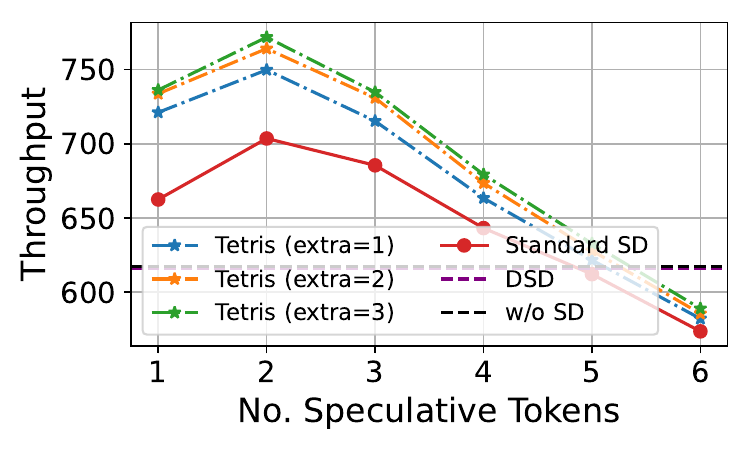} &
    \includegraphics[width=0.31\linewidth]{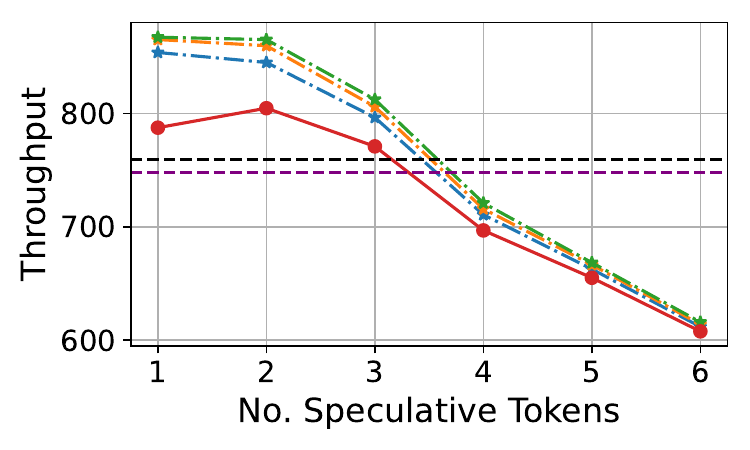} &
    \includegraphics[width=0.31\linewidth]{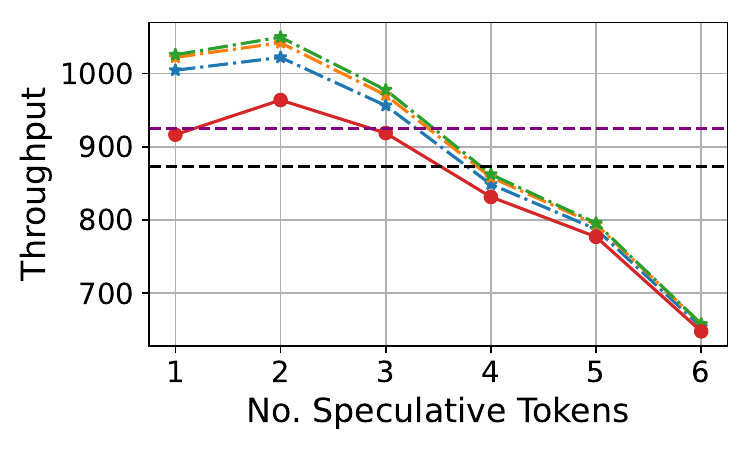} \\
    & (a) \hspace{5mm} $\hat{\gG}^{(\textit{TER})}$$\uparrow$ 9.70\% & (b) \hspace{5mm} $\hat{\gG}^{(\textit{TER})}$$\uparrow$ 7.79\% & (c) \hspace{5mm} $\hat{\gG}^{(\textit{TER})}$$\uparrow$ 8.92\% \\
    
    \rotatebox{90}{\parbox{2.5cm}{\centering \hspace{8mm}\textbf{Setting 2}}}  & \includegraphics[width=0.31\linewidth]{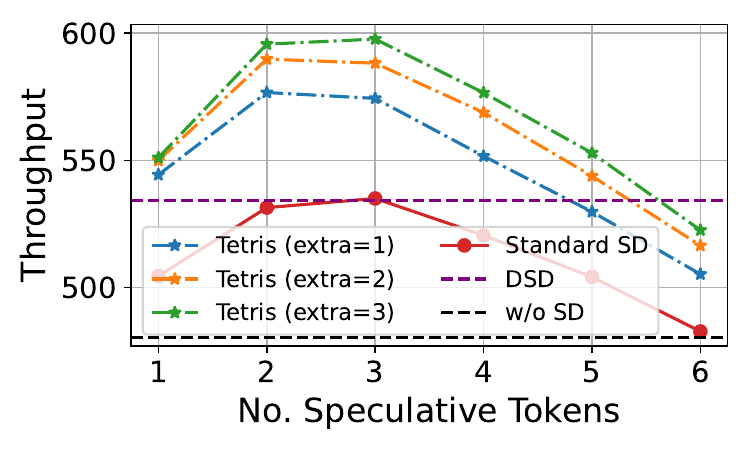} & \includegraphics[width=0.31\linewidth]{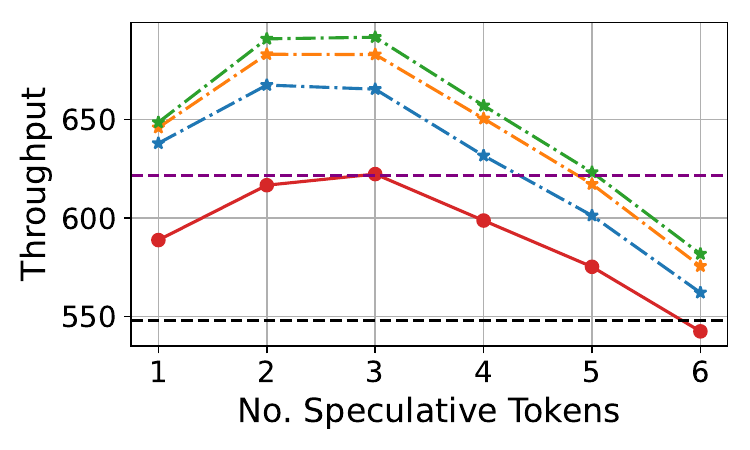} & \includegraphics[width=0.31\linewidth]{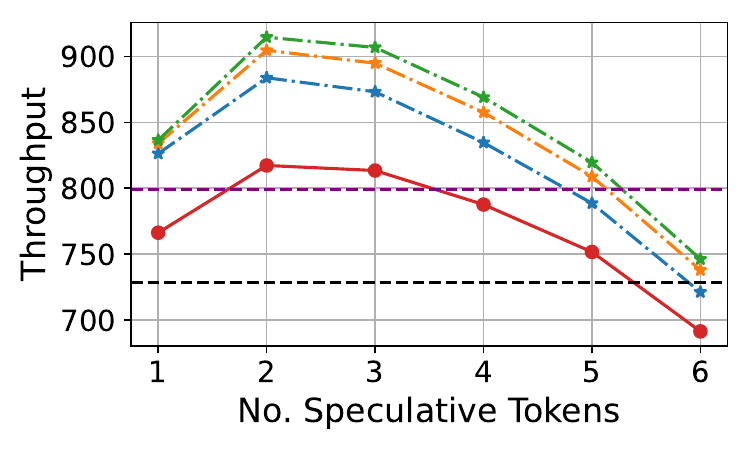}  \\
    & (d) \hspace{5mm} $\hat{\gG}^{(\textit{TER})}$$\uparrow$ 11.70\% & (e) \hspace{5mm} $\hat{\gG}^{(\textit{TER})}$$\uparrow$ 11.17\% & (f) \hspace{5mm} $\hat{\gG}^{(\textit{TER})}$$\uparrow$ 11.91\% \\
    
    \rotatebox{90}{\parbox{2.5cm}{\centering \hspace{8mm}\textbf{Setting 3}}}  & \includegraphics[width=0.31\linewidth]{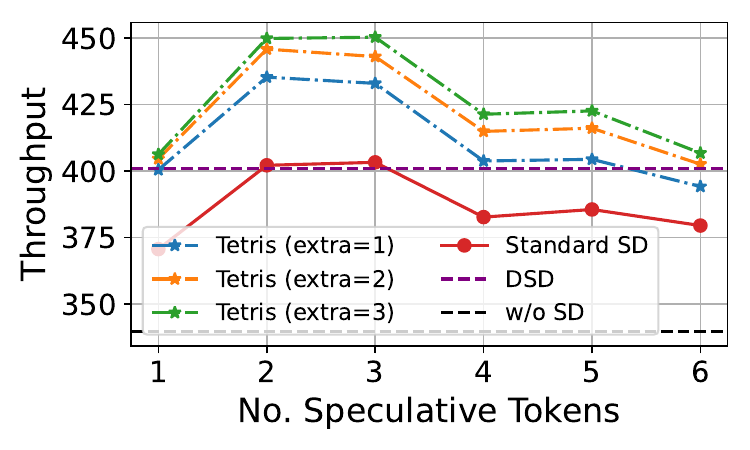} & 
    \includegraphics[width=0.31\linewidth]{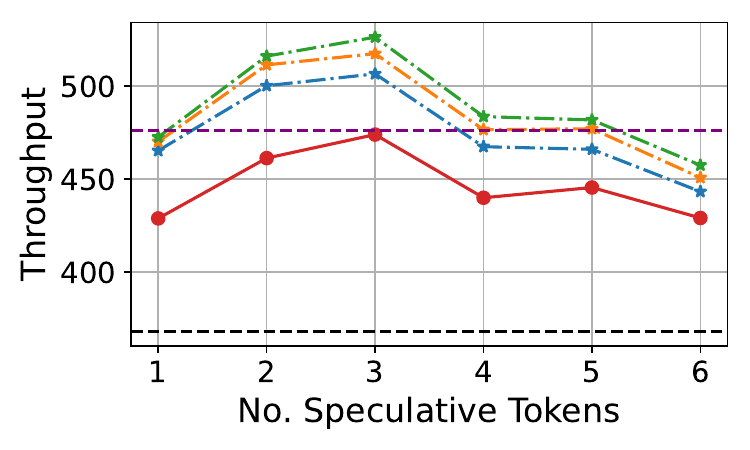} &
    \includegraphics[width=0.31\linewidth]{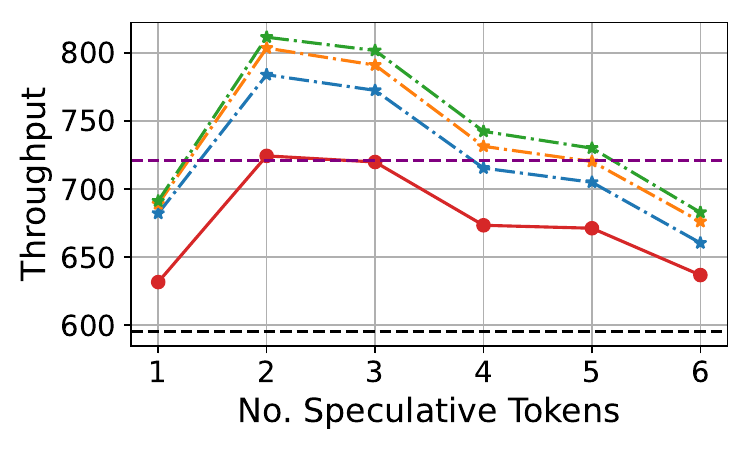} \\
    & (g) \hspace{5mm} $\hat{\gG}^{(\textit{TER})}$$\uparrow$ 11.67\% & (h) \hspace{5mm} $\hat{\gG}^{(\textit{TER})}$$\uparrow$ 10.53\% & (i) \hspace{5mm} $\hat{\gG}^{(\textit{TER})}$$\uparrow$ 12.04\% \\
\end{tabular}
    \caption{Mean projected throughput $\hat{\gG}^{(\textit{TER})}$ comparison for various methods across experimental settings. $\uparrow$ indicates the improvement from the best baseline method. The reported numbers reflect the mean over 3 independent trials.}
    \label{fig:projected-throughput-all-appendix}
\end{figure}

We provide an extended discussion on the improvement of end-of-end latency from~\cref{sec:evaluation}.
In~\cref{fig:latency-all-appendix}, we show the plots for the end-to-end latency over all speculative decoding configurations and settings used in the paper.
\alg{} consistently outperforms the existing baselines and achieves up to 6.13\% improvement over the best baseline and up to 9.32\% maximum gap over standard SD.
Therefore, \alg{} has demonstrated to effectively reduce end-to-end request latency, which is also essential for enhancing the user experience with LLM inference service providers.

\subsection{Plots for Projected Improvement based on TER}\label{app:projected-improvement-plots}

Complementary to~\cref{tab:potential}, which contains the numerical results for the projected improvement of \alg{} in terms of the projected throughput $\hat{\gG}^{(\textit{TER})}$, we also show the plots in~\cref{fig:projected-throughput-all-appendix} to visually illustrate the effectiveness of our method. 
The dotted lines for \alg{} (drawn in blue, orange, and green) represent the projected throughput calculated based on the throughput of the standard SD and also the \alg{}'s improvement in terms of target efficiency rate (TER, as defined in~\cref{eq:TER}).
We note that these improvement numbers are theoretically computed and are not yet realizable in empirical settings due to the lack of parallelized pipeline implementations of speculative decoding in \vllm{}.

\subsection{Extension to Medusa}\label{app:medusa-extension}

We evaluate the top-1 proposal version (i.e., only draft the most likely token for each position) of Medusa and its integration with \alg{}.
As the Medusa model outputs multiple subsequent tokens in a single forward pass,\footnote{We use a modified implementation of Medusa in \vllm{} to ensure a fixed forward pass time.} we leverage this feature to produce extra draft tokens for \alg.
We show the results in~\cref{tab:ablation-medusa}.
We achieved a throughput improvement of 3.19\% as compared to the baseline Medusa.
The development of such multi-token prediction models, including models like EAGLE~\citep{li2024eagle} and DeepSeek-V3~\citep{deepseekai2024deepseekv3technicalreport} presents further potential for \alg{} to achieve greater speedups. Other improvements in engineering, including using tree-decoding and using a larger target model also potentially further boost the speedup.

\begin{table}[!ht]
\centering
\caption{Mean total throughput ($\pm$ standard deviation) for the ablation study of \alg{} extension to Medusa over three independent trials. The integration of \alg{} with Medusa further improves the total throughput.}
\label{tab:ablation-medusa}
\resizebox{0.9\linewidth}{!}{
\begin{tabular}{c|cccc}
\toprule
\textbf{No. Speculative Tokens} &
\multirow{1}{*}{\textbf{\alg{} (extra=1)}} & \multicolumn{1}{c}{\textbf{\alg{} (extra=2)}} & \multicolumn{1}{c}{\textbf{\alg{} (extra=3)}} & \multicolumn{1}{c}{\textbf{Baseline Medusa}}  \\ 
\midrule
1 & 591.26$\pm$0.46 & 590.83$\pm$8.30 & 586.47$\pm$3.66 &  572.97$\pm$1.79 \\
2 & 571.05$\pm$0.80 & 568.82$\pm$6.52 & 571.95$\pm$1.06 & 563.94$\pm$2.95 \\
\midrule
\textit{Best} & 591.26 & 590.83 & 586.47 & 572.97\\
\bottomrule
\end{tabular}}
\end{table}

\subsection{Improvement in Verification Success Rate}\label{app:improve-VSR}

As an ablation study, we also illustrate the improvement of \alg{} in terms of VSR (as defined in \cref{eq:VSR}), which is an important measure of the effectiveness of speculative decoding.
We show in~\cref{fig:VSR-all-appendix} that the maximum gap between \alg{} and standard SD in terms of VSR is consistently above 20\% and reaching over 30\% in some instances.
This validates the significant effect of \alg{} in selecting draft tokens that are most likely to be accepted by the target model without exceeding the system capacity of the server.
However, it is worth noting that this improvement in VSR does not translate entirely to an increment in total throughput or a reduction in end-to-end latency.
This is because the throughput in practice also depends on the running time of the draft model (especially when the speculative decoding pipeline is sequential, as discussed in~\cref{sec:parallel-implementation}), and VSR does not account for the generation of the bonus token (which takes up a portion of the generated tokens).

\begin{figure}[t]
    \centering
    \setlength{\tabcolsep}{1pt} 
    \begin{tabular}{cccc}
    & \hspace{8mm}\textbf{ShareGPT} & \hspace{8mm}\textbf{Arena} & \hspace{8mm}\textbf{Tough} \\
    \rotatebox{90}{\parbox{2.5cm}{\centering \hspace{8mm}\textbf{Setting 1}}} & \includegraphics[width=0.31\linewidth]{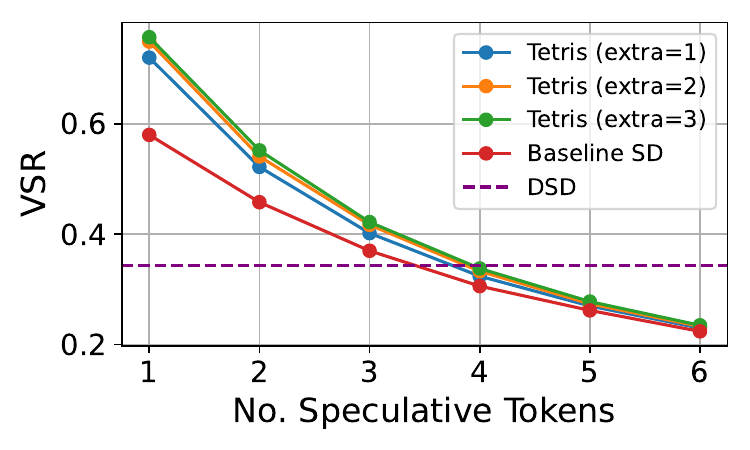} &
    \includegraphics[width=0.31\linewidth]{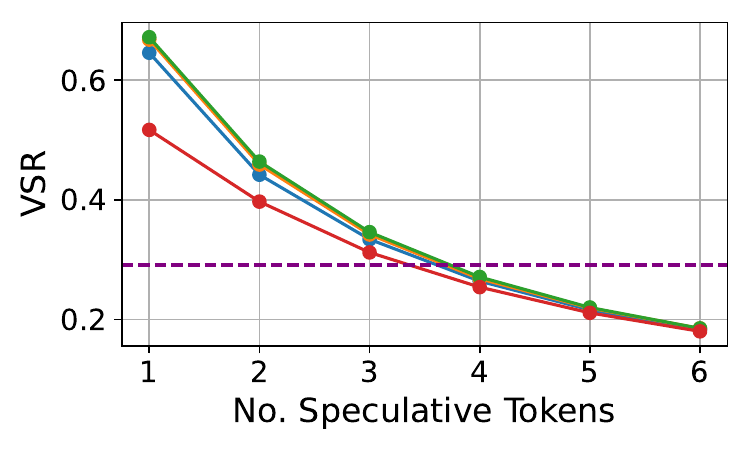} &
    \includegraphics[width=0.31\linewidth]{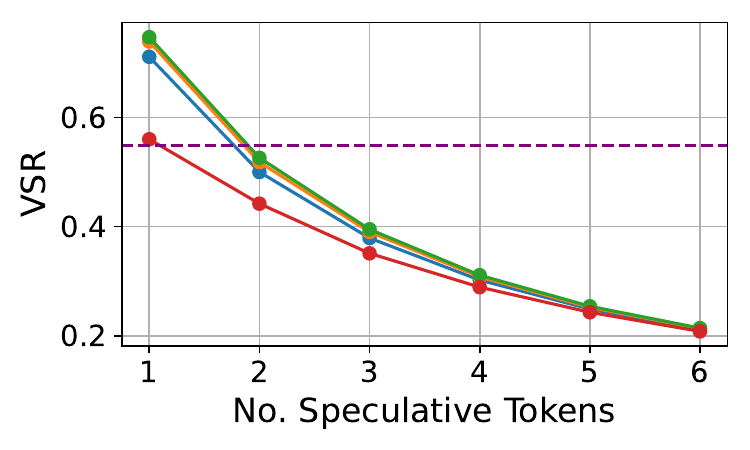} \\
    & (a) \hspace{5mm} $\Delta$ 30.52\% & (b) \hspace{5mm} $\Delta$ 29.98\% & (c) \hspace{5mm} $\Delta$ 33.39\% \\
    
    \rotatebox{90}{\parbox{2.5cm}{\centering \hspace{8mm}\textbf{Setting 2}}}  & \includegraphics[width=0.31\linewidth]{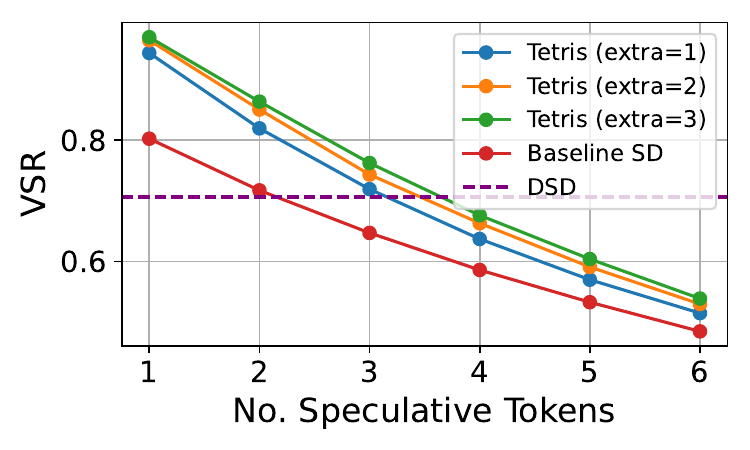} & \includegraphics[width=0.31\linewidth]{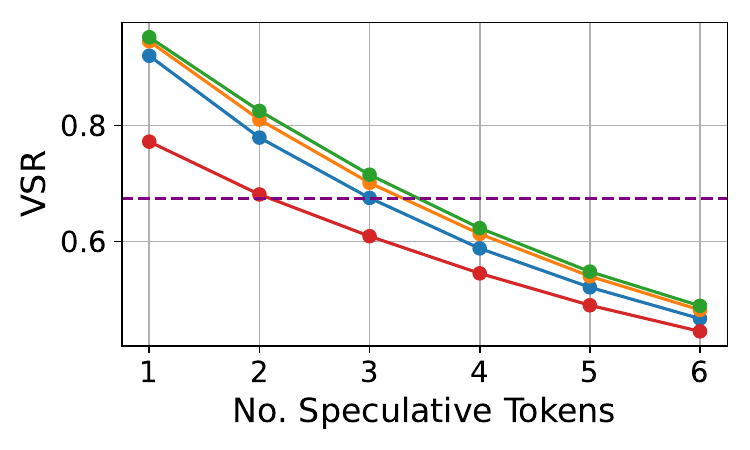} & \includegraphics[width=0.31\linewidth]{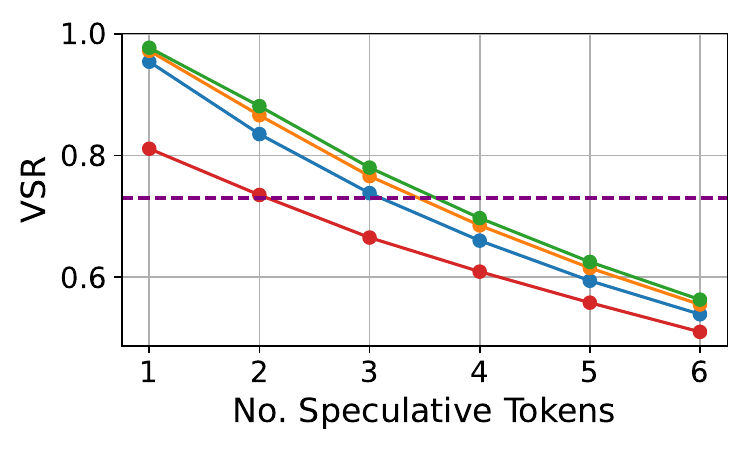}  \\
    & (d) \hspace{5mm} $\Delta$ 20.82\% & (e) \hspace{5mm} $\Delta$ 23.32\% & (f) \hspace{5mm} $\Delta$ 20.47\% \\
    
    \rotatebox{90}{\parbox{2.5cm}{\centering \hspace{8mm}\textbf{Setting 3}}}  & \includegraphics[width=0.31\linewidth]{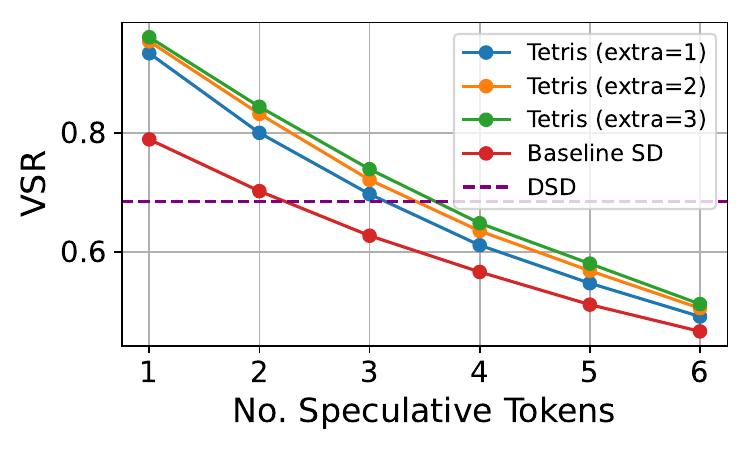} & 
    \includegraphics[width=0.31\linewidth]{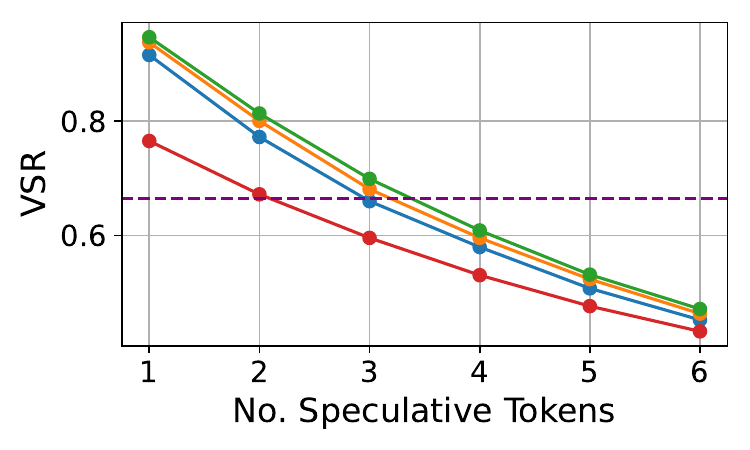} &
    \includegraphics[width=0.31\linewidth]{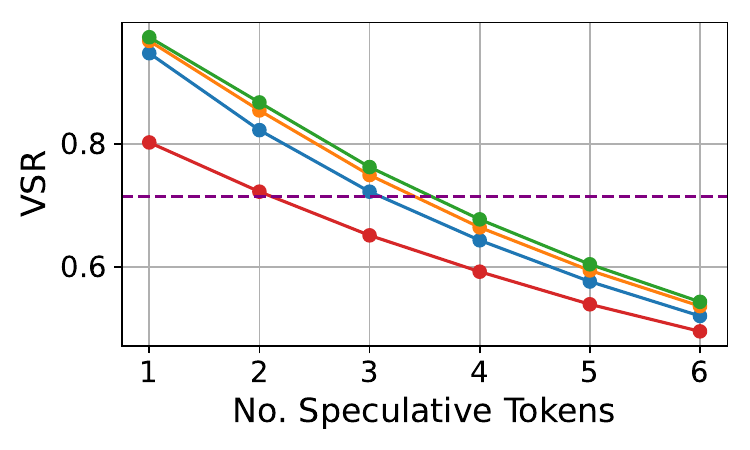} \\
    & (g) \hspace{5mm}  $\Delta$ 21.80\% & (h) \hspace{5mm} $\Delta$ 23.66\% & (i) \hspace{5mm} $\Delta$ 21.32\% \\
\end{tabular}
    \caption{The verification success rate comparison for various methods across experimental settings. $\Delta$ indicates the maximum gap between \alg{} and standard SD. The reported numbers reflect the mean over 3 independent trials.}
    \label{fig:VSR-all-appendix}
\end{figure}

\subsection{The Effect of Batch Size on \alg{} Performance}\label{app:batch-size}

Theoretically speaking, a larger batch size creates more possible combinations for draft token selection by \alg.
Therefore, \alg{} is likely to perform better in a speculative decoding server that processes a larger batch of requests concurrently.
In~\cref{fig:batch-size-TER-appendix}, we show a visual illustration of the verification success rate (VSR) and target efficiency rate (TER) (as defined in~\cref{eq:VSR} and~\cref{eq:TER}, respectively).

In setting 2 (draft model: Llama-1B-Instruct-FP8, target model: Llama-70B-Instruct), we observe a significant increase in VSR and TER when the batch size is increased to 64.
However, batch sizes of 16 and 32 have similar VSR and TER values.

In setting 3 (draft model: Llama-1B-Instruct-FP8, target model: Llama-405B-Instruct-FP8), we do not observe a significant change in VSR and TER, suggesting that the way that the batch size affects performance is highly dependent on the specific draft-target combination, too.

Overall, we expect a more significant improvement in the performance of adopting \alg{} by LLM inference service providers with larger capacities to handle a larger number of concurrent requests.

\begin{figure}[t]
    \centering
    \setlength{\tabcolsep}{1pt} 
    \begin{tabular}{ccc}
    & \hspace{8mm}\textbf{Setting 2} & \hspace{8mm}\textbf{Setting 3}  \\
    \rotatebox{90}{{\centering\hspace{1.8cm}\textbf{ShareGPT}}} & \includegraphics[width=0.47\linewidth]{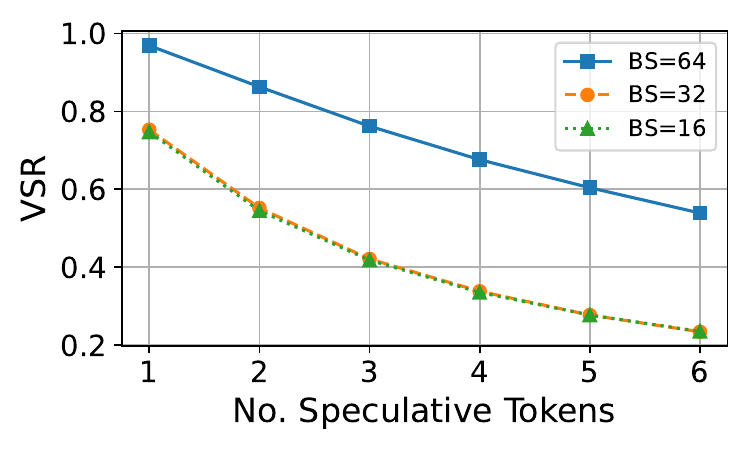} &
    \includegraphics[width=0.47\linewidth]{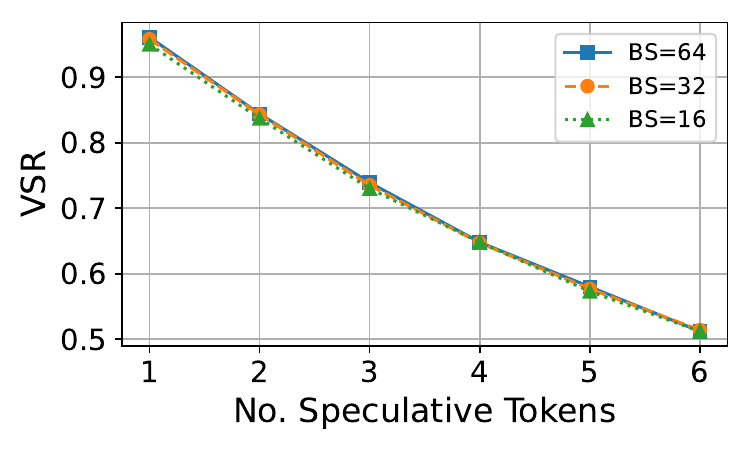} \\
    
    \rotatebox{90}{{\centering\hspace{1.8cm}\textbf{ShareGPT}}} & \includegraphics[width=0.47\linewidth]{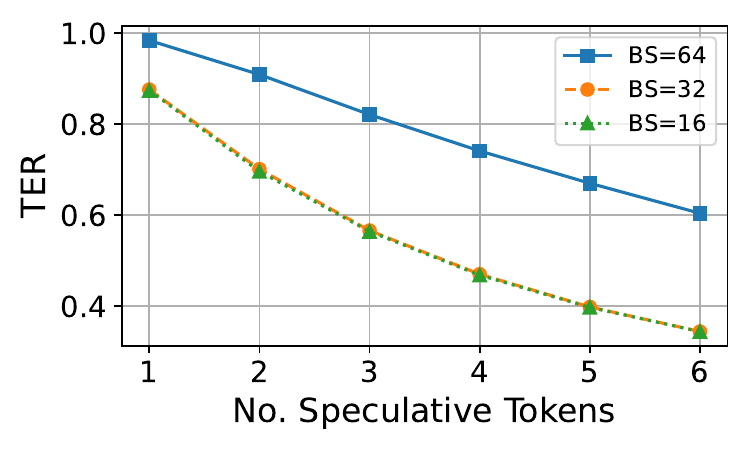} &
    \includegraphics[width=0.47\linewidth]{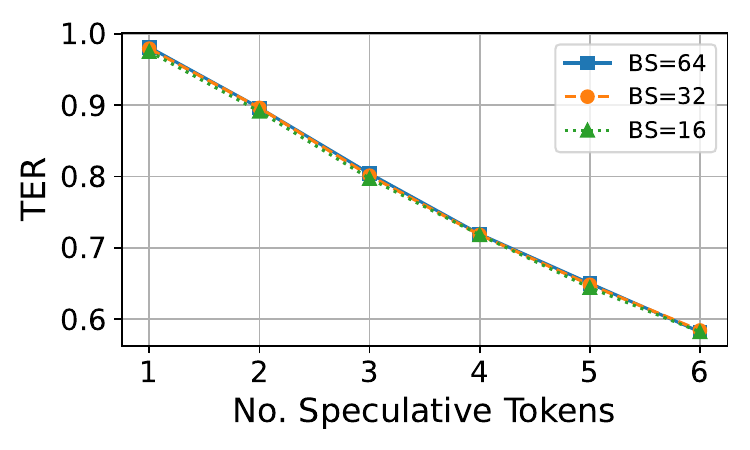} \\
    
\end{tabular}
    \caption{The change in the verification success rate (VSR) and target efficiency rate (TER) when we vary the batch size (BS) from 64 to 32 and 16. The reported numbers reflect the mean over 3 independent trials.}
    \label{fig:batch-size-TER-appendix}
\end{figure}

\section{Broader Impacts}\label{app:broader-impacts}

While this research work is primarily foundational, focusing on computational performance, the resulting increase in inference speed and efficiency of Large Language Models (LLMs) could indirectly contribute to certain societal risks associated with LLMs. Making LLM inference faster and cheaper lowers the barrier to deploying these models at scale. Consequently, this could potentially accelerate the proliferation of LLM-generated content, increasing the risks of misuse such as the large-scale generation of disinformation, spam, or fake online personas, if the underlying models are deployed without adequate safeguards. 

Mitigation strategies depend on responsible deployment. Developers using \alg{} should employ robust safety measures, bias detection, and content filtering for the LLMs being served. Importantly, the efficiency gains from \alg{} could be leveraged positively to make computational overhead for safety checks, alignment techniques, or bias mitigation measures more feasible during inference.
\\
    \hrule height 0.5mm

\end{document}